\documentclass[12pt,a4paper]{article}

\usepackage[margin=1in]{geometry} 
\usepackage[utf8]{inputenc} 
\usepackage{authblk} 
\usepackage{hyperref} 
\usepackage[numbers,sort]{natbib} 
\usepackage{color}
\usepackage{array,multirow}

\usepackage{enumitem}
\setlist{itemsep=0pt}

\usepackage{amssymb,amsmath,amsthm}
\usepackage{mathtools}
\usepackage{graphicx}
\usepackage{subcaption}
\usepackage[ruled]{algorithm2e}
\usepackage[capitalise]{cleveref}

\theoremstyle{definition}
\newtheorem{theorem}{Theorem}[section]
\newtheorem{proposition}{Proposition}
\newtheorem{lemma}{Lemma}

\newtheorem{conjecture}[theorem]{Conjecture}

\theoremstyle{remark}
\newtheorem{remark}[theorem]{Remark}

\newcommand{\keywords}[1]{\textbf{Keywords:} #1}

\newcommand{\R}{\mathbb{R}}
\newcommand{\E}{\mathbb{E}}

\DeclareMathOperator{\Var}{Var}

\DeclareMathOperator{\Dkl}{D_\mathrm{KL}}

\renewcommand{\d}{\mathrm{d}}

\usepackage{xargs}                      
\usepackage[pdftex,dvipsnames]{xcolor}  
\usepackage[colorinlistoftodos,prependcaption,textsize=tiny,textwidth=0.85in]{todonotes}
\newcommandx{\oz}[2][1=]{\todo[linecolor=OliveGreen,backgroundcolor=OliveGreen!25,bordercolor=OliveGreen,#1]{#2}}

\usepackage[normalem]{ulem} 

\newcommand{\lb}{\left(}
\newcommand{\rb}{\right)}
\newcommand{\trans}{\mathsf{T}}

\newcommand{\mcN}{\mathcal{N}}

\newcommand{\opt}{\mathrm{opt}}
\newcommand{\Ent}{\textsc{Ent}}

\newcommand{\mcJ}{\mathcal{J}}
\newcommand{\x}{x}
\renewcommand{\d}{\mathrm{d}}

\newcommand{\df}{\mathrm{df}}
\newcommand{\KL}{\mathrm{KL}}
\newcommand{\Hell}{\mathrm{Hell}}

\newcommand{\pioptKLr}{\widetilde{\pi}^{\KL}}

\DeclareMathOperator{\trace}{tr\,}

\DeclareMathOperator{\dhell}{d_{\Hell}^2}

\begin{document}

\title{Sharp detection of low-dimensional structure in probability measures via dimensional logarithmic Sobolev inequalities}

\author{Matthew T.C. Li\footnote{{Center for Computational Science and Engineering}, {MIT}, {{USA}}},
Tiangang Cui\footnote{{School of Mathematics and Statistics}, {University of Sydney}, {{Australia}}},
Fengyi Li$^{*}$,Youssef Marzouk$^{*}$ and
Olivier Zahm\footnote{{Universit\'e Grenoble Alpes}, {Inria, CNRS, Grenoble INP, LJK}, {Grenoble}, {France}}
}
\date{}

\maketitle

\begin{abstract}
 Identifying low-dimensional structure in high-dimensional probability measures is an essential pre-processing step for efficient sampling.
To identify this structure, we approximate the target measure as a perturbation of an arbitrary reference measure along a few directions in $\R^{d}$. These directions are determined by minimizing an upper bound on the Kullback–Leibler (KL) divergence between the target and its approximation. Our contribution improves upon previous works by leveraging \emph{dimensional} logarithmic Sobolev inequalities (LSI) to refine the bound on the KL divergence. These inequalities lead to a uniformly tighter bound on the KL divergence, thereby enhancing the identification of the most significant perturbation directions.
In particular, when the target and reference are both Gaussian, minimizing the resulting bound is equivalent to minimizing the KL divergence.
We further demonstrate the applicability of this analysis to the squared Hellinger distance, where analogous reasoning shows that the \emph{dimensional} Poincar\'e inequality offers improved bounds.
\end{abstract}


\section{Introduction} \label{sec:intro}

{High-dimensional probability measures often exhibit an inherent low-dimensional structure that can be leveraged to improve computational efficiency in various tasks, including sampling. Our objective is to identify such a low-dimensional structure by exploiting the fact that a given target measure $\pi$ is close, in some sense, to a low-dimensional perturbation of an arbitrary reference measure $\mu$.
When properly identified, this structure can significantly enhance the performance of sampling algorithms, which is known to be challenging as the dimension~$d$ increases  \cite{roberts1998optimal,pillai2012,andrieu2022explicit,mangoubi2018dimensionally,Snyder_Bengtsson_Bickel_Anderson_2008, Rebeschini_van_Handel_2015}.}
Typically, $\mu$ takes the form of a product measure, a Gaussian measure, or a nonlinear transformation of it, such as those used in GANs or other generative models~\cite{GANs:2014}.
Depending on the application, $\mu$ can either be specified by domain experts, or $\mu$ can be constructed from empirical samples. {Throughout this manuscript we consider only probability measures admitting a Lebesgue density, and we further assume that $\tfrac{\d\pi}{\d\mu}(x) \propto \ell(x)$ for some smooth positive function $\ell:\R^d\rightarrow\R$.}

In many realistic scenarios, $\pi$ is close to $\mu$ in that both measures essentially differ only along a few significant dimensions of $\R^d$. This suggests that $\pi$ can be approximated by a probability measure $\widetilde\pi$ defined as
\begin{equation}\label{eq:piTildeIntro}
    \d\widetilde\pi(x) \propto \ell_r(U_r^\top x)\d\mu(x),
\end{equation}
with the \emph{profile function} $\ell_r:\R^r\rightarrow\R_+$,  a low-dimensional function of $r\ll d$ variables, acting on the range of the \emph{linear features} $U_r \in \R^{d \times r}$, both to be determined.
{Without loss of generality we consider $U_r$ with orthonormal columns, so that $U_r^\top U_r = I_r$ is the identity matrix of size~$r$, since any scalings or dilations  can be absorbed into the definition of~$\ell_r$.}
Such structured approximation --- which we interchangeably describe as \emph{dimension reduction} for the probability measure~$\pi$ --- bears relevance to at least the following contexts:
\begin{itemize}
\setlength{\itemsep}{1em}
 \item In Bayesian inverse problems \cite{stuartActa,KaipioSomersalo} one is interested in the posterior measure $\d\pi(x) \propto \ell(x) \d\mu(x)$, where $\mu$ denotes the prior measure of the model parameters~$x \in \R^d$, and $\ell$ the likelihood function associated with observed data. The approximation \eqref{eq:piTildeIntro} replaces the likelihood~$\ell$ with a ridge function $x\mapsto \ell_r(U_r^\top x)$, and the linear features $U_r$ capture low-dimensionality in the prior-to-posterior update. {These features correspond to the directions of the parameter $x$ for which data are most informative.}

 \item In generative modeling, probability measures~$\mu$
 are modeled as the push-forward of lower-dimensional measures~$\nu$ in~$\R^\kappa$ with $\kappa \ll d$ latent coordinates. The transport map $T: \R^\kappa \to \R^d$, so that $\mu = T_\# \nu$, is learned from empirical samples~\cite{GANs:2014}. Given a likelihood model specified in the nominal $d$~dimensional parameter space, one then wishes to perform conditional sampling, which requires only sampling $\d\pi(z) = (\ell \circ T) (z) \d\nu(z)$ in the latent space~\cite{patelSolutionPhysicsbasedBayesian2022a}. Here, the approximation~\eqref{eq:piTildeIntro} replaces the composed likelihood $\ell \circ T$ with a ridge function capturing the low-dimensional prior-to-posterior updates, further complementing the dimensionality reduction brought on by the latent space compression.

 \item In molecular dynamics, $\d\pi(x)\propto\exp(-\beta V(x))\d x$ is a Boltzmann distribution associated with a molecular system with energy potential $x\mapsto V(x)$.
 Such distributions are often multi-modal, with each mode corresponding to a stable physical conformation.
{If $\pi$ is well-approximated by a measure $\widetilde\pi$ as in \eqref{eq:piTildeIntro} using a unimodal reference measure~$\mu$, then the features $U_r$ have a physical interpretation as \emph{reaction coordinates}: directions in configuration space from which the multi-modality of $\pi$ emerges. Discovering meaningful reaction coordinates is an active research field, see e.g., \cite{pillaud-vivienStatisticalEstimationPoincar2019}.}

\end{itemize}

With some abuse of notation, we denote the measures in \eqref{eq:piTildeIntro} by $\widetilde\pi(U_r,\ell_r)$.
Optimal approximations within this class can be determined by minimizing the Amari $\alpha$-divergence~\cite{Amari_2009}
\begin{equation}
\label{eq:loss}
\min_{\substack{U_r \in \R^{d \times r} \\ U_r^\top U_r = I_r}} \, \min_{\ell_r : \R^r \to \R_+} D_\alpha(\pi\,||\,\widetilde\pi(U_r,\ell_r)) \;,
\end{equation}
though other probability divergences or metrics are also of interest.
Choosing $\alpha \in \R$ interpolates the Kullback-Leibler divergence ($\alpha=1$), the squared Hellinger distance ($\alpha=1/2$), the $\chi^2$~divergence ($\alpha=2$), {as well as each divergence with the reverse ordering of arguments}\footnote{{Here and throughout, our indexing follows the definition of $\alpha$-divergences used in~\cite{Chafai_2004, Bolley_Gentil_2010}, which matches the original $\alpha'$-divergence in~\cite{Amari_2009} by a multiplicative factor, and by setting $\alpha = \tfrac{1}{2}(1-\alpha')$.}}.
Since \cite[Thm 1]{Li_Marzouk_Zahm_2023} shows that the optimal profile function~$\ell_r^*=\ell_r^*(U_r)$ admits a closed-form solution depending on $\alpha$ and $U_r$, the problem in~\eqref{eq:loss} reduces to a finite dimensional optimization program for features~$U_r$ minimizing $D_\alpha(\pi\,||\,\widetilde\pi(U_r,\ell_r^*(U_r)))$. Nevertheless, this objective function requires access to the normalizing constants of~$\widetilde\pi(U_r,\ell_r^*(U_r))$ and~$\pi$, and gradients thereof, and thus remains computationally intractable in higher dimensions.

Zahm et al.~\cite{ZCLSM22} and subsequent works circumvent this problem by \emph{majorizing} the inner optimization problem in~\eqref{eq:loss} following
\begin{equation}\label{eq:majorant}
  \min_{\ell_r : \R^r \to \R_+} D_\alpha(\pi\,||\,\widetilde\pi(U_r,\ell_r)) \leq \mcJ_\alpha(U_r) \;.
\end{equation}
The majorant $U_r \mapsto \mcJ_\alpha(U_r)$, whose dependence on~$\pi$ we drop from our notation, is constructed to be both easier to evaluate numerically and to minimize with respect to~$U_r$ compared to~\eqref{eq:loss}.
{The minimizer~$U_r^*$ of the right-hand side of \eqref{eq:majorant} is not guaranteed to be a global minimizer of the Amari $\alpha$-divergence, but one can guarantee
any desired error tolerance $D_\alpha(\pi\,||\,\widetilde\pi(U_r^*,\ell_r^*(U_r^*))) \leq \epsilon$ by selecting sufficiently large~$r \in \{1, \cdots, d \}$ which ensures~$\mcJ_\alpha(U_r^*) \leq \epsilon$.
In many practical situations, it has been observed that it suffices to choose~$r \ll d$ to achieve ~$\epsilon \ll 1$.
For this reason, we refer to this as a \emph{computationally certifiable approximation}, analogous to the use of dual certificates in optimization.
}

These majorizations are constructed by relating each $\alpha$-divergence to a functional inequality satisfied by the reference measure~$\mu$, and they only require the Radon-Nikodym derivative $\d\pi/\d\mu$ to be a smooth function. The seminal contribution of Zahm et al.~\cite{ZCLSM22} related the Kullback-Leibler (KL) divergence to the \emph{logarithmic Sobolev inequality} (LSI, \cite{stamInequalitiesSatisfiedQuantities1959a, Gross_1975}), while  \cite{Cui_Tong_2021,Cui_Dolgov_Zahm_2022,flock2023certified} related the squared Hellinger distance to the \emph{Poincar\'e inequality}.
Later~\cite{Li_Marzouk_Zahm_2023} generalized these results by relating each Amari $\alpha$-divergence, with $\alpha \in (0,1]$, to a corresponding \emph{$\phi$-Sobolev inequality}~\cite{Chafai_2004, Bolley_Gentil_2010}.
Surprisingly, the minimizer of $U_r^*$ of $U_r \mapsto \mcJ_\alpha(U_r)$ is identical for all $\alpha \in (0,1]$, and is given by the matrix whose  columns contain the dominant eigenvectors of the matrix
\begin{equation}\label{eq:defH}
 H(\pi||\mu) =
 \E_{\pi}\left[\nabla \ln \left(\frac{\d\pi}{\d\mu}\right)\nabla \ln \left(\frac{\d\pi}{\d\mu}\right)^\top\right] \;,
\end{equation}
also known as the \emph{relative Fisher information matrix of $\pi$ to $\mu$}.

An interesting open question is whether better majorants---and thus features with better certifiable approximation guarantees---can be derived for specific~$\alpha \in (0,1]$. A heuristic motivation for exploring alternative functional inequalities stems from the observation that the $\phi$-Sobolev inequalities, including the LSI or Poincar\'e inequality,  are `dimensionally independent', and thus formally valid even for infinite-dimensional distributions~\cite{Bakry2014, Chafai_2004}. However, in constructing majorants for the purpose of dimension \emph{reduction} this dimensional independence property is not explicitly used.

To that end, our main contribution in Theorem~\ref{thm:dimCDR} demonstrates that an improved majorant and minorant
\begin{equation}
\label{eq:dimCDRbound_intro}
\mcJ_\KL^\uparrow(U_r) \leq
\min_{\ell_r : \R^r \to \R_+}
D_\KL(\pi\,||\,\widetilde\pi(U_r,\ell_r))
\leq \mcJ_\KL^\downarrow(U_r)
\end{equation}
can be derived for the KL divergence using the \emph{dimensional logarithmic Sobolev inequality}~\cite{Bakry_Ledoux_2006, Dembo_1990},
where above we have suppressed implicit dependencies of the functions $U_r \mapsto \mathcal{J}^{\downarrow (\uparrow)}_\KL(U_r)$ on~$\pi$ for notational simplicity. This functional inequality improves upon the LSI by explicitly accounting for the nominal dimension~$d$ of the reference measure~$\mu$. Consequently, we can show that
$\mcJ^\downarrow_\KL(U_r)  \leq \mcJ_{\alpha =1}(U_r)$ for all $U_r$ with orthonormal columns,
demonstrating that the new majorant uniformly improves on the bound obtained in \cite{ZCLSM22} from the LSI by as much as an exponential factor.
{Surprisingly, in Proposition \ref{thm:dimCDRexactgauss} we show that
the dimensional majorant~$\mcJ^\downarrow_\KL(U_r)$
is equal to
the KL divergence~\eqref{eq:dimCDRbound_intro}
when $\pi$ is a Gaussian measure, even though the dimensional LSI was not originally derived with dimension reduction in mind.}
While the minorant $\mcJ_\KL^\uparrow(\cdot)$ is not used to determine the minimizer $U_r^\downarrow$ of the right-hand side of~\eqref{eq:dimCDRbound_intro}, it provides a useful certificate for an achievable best case error.

Evaluating and optimizing the improved majorant $\mcJ_\KL^\downarrow(\cdot)$ requires computing
\begin{align}\label{eq:defHandM}
 H(\pi) &= \E_{\pi}\left[\nabla \ln \left(\frac{\d\pi}{\d x}\right)\nabla \ln \left(\frac{\d\pi}{\d x}\right)^\top\right]
 \quad \text{and} \quad
 M(\pi) = \E_{\pi}[ XX^\top],
\end{align}
where $H(\pi)$ is the Fisher information matrix of $\pi$,
and $M(\pi)$ is the second moment matrix of $\pi$.
Evaluating the minorant $\mcJ_\KL^\uparrow(\cdot)$ requires in addition the computation of $m(\pi)=\E_\pi[X]$, the mean of the measure $\pi$.
{The numerical estimation of $H(\pi)$, $M(\pi)$ and $m(\pi)$ is a classical problem in computational statistics, see \cite{ZCLSM22,haario2001adaptive,uribe2021cross,chen2020projected}.}
To the best of our knowledge,
there is no closed-form expression for the global minimizer of $\mcJ_\KL^\downarrow(\cdot)$. Nevertheless, we observe numerically that first order gradient descent methods converge to the same minimizer irrespective of the initialization. We conjecture that $\mcJ_\KL^\downarrow(\cdot)$ is \emph{benignly non-convex}~\cite{sunWhenAreNonconvex2016} on the Grassmann manifold,
but we cannot prove this and leave this as a problem of independent interest for future work.

Alternatively, in Theorem~\ref{thm:dimCDR_tilted}  we consider the reference measure~$\mu = \mcN(m(\pi),\, C(\pi))$ matching the mean~$m(\pi)$ and covariance~$C(\pi)$ of the target measure~$\pi$. In this setting, the global minimizer of the majorization derived from the dimensional LSI admits a closed form solution.
Intriguingly, this bound suggests that relative deviations between $\pi$ and its moment-matched Gaussian approximation are reflected in the relative changes
between the Fisher information matrix~$H(\pi)$ and the precision matrix~$C(\pi)^{-1}$, in that
\begin{equation}
 {\min_{\ell_r : \R^r \to \R_+} \Dkl(\pi || \pioptKLr(U_r^\downarrow,\ell_r)) }
 \leq \frac{1}{2}\sum_{k=r+1}^d \ln (   \lambda_k(H(\pi) , C(\pi)^{-1})  ),
\end{equation}
where $\lambda_k(A,B) $ is the $k$-th largest eigenvalue of the generalized eigenvalue problem $Av = \lambda Bv$, and the columns of $U_r^\downarrow$ contain the first~$r$ generalized eigenvectors.

For
Bayesian inverse problems, an alternative notion of low-dimensional structure can also be considered. Taking the target measure to be the posterior~$\pi_{X \mid Y}$ of state~$X$ conditioned on observations~$Y$,
\cite{Cui_Zahm_2021} show that the LSI can be used to determine features~$V_r^*$ along which the posterior differs most from the prior~$\mu_X$ for the \emph{average} realization of data~$Y$.
We show in Theorem~\ref{thm:dimCDRdatafree} that the dimensional LSI further improves on these bounds. Specifically, we recover the same features $V_r^*$ as in~\cite{Cui_Zahm_2021}, corresponding to the leading eigenvectors of
$H_\mathrm{df}(\pi_{X \mid Y}) := \E_{X,Y}[\nabla_x \ln \pi_{x \mid y}\nabla_x \ln \pi_{x \mid y}^\top]$, but
with the improved certificate of approximation
\begin{equation}
\label{eq:dimCDRdatafree_intro}
 \E_Y\left[ \Dkl(\pi_{X \mid Y} \,||\, \widetilde{\pi}_{X \mid Y}(V_r^*)  \right] \leq \frac{1}{2} \sum_{k=r+1}^d \ln(1 + \lambda_k(H_\mathrm{df}(\pi_{X \mid Y}))),
\end{equation}
where $\lambda_k(A)$ denotes the $k$-th largest eigenvalue of $A$. Practically, \eqref{eq:dimCDRdatafree_intro} implies that significantly fewer features are necessary to attain a certifiable error level than previously believed.
The form of the certificate~\eqref{eq:dimCDRdatafree_intro} also highlights a connection to the \emph{expected information gain} (EIG) appearing in optimal experimental design~\cite{alexanderianBayesianDOptimalExperimental2016,ghattasLearningPhysicsbasedModels2021,huanOEDacta2024}.
The application of functional inequalities to optimal experimental design is also a topic of recent interest~\cite{liNonlinearBayesianOptimal2024, qiao2024oed}.

Lastly, we also explore structured approximations in~\eqref{eq:loss} with the squared Hellinger distance corresponding to~$\alpha = \tfrac{1}{2}$.  Motivated by the same heuristic considerations behind exploring the dimensional LSI, we show in Theorem~\ref{thm:dimCDRhell} that the \emph{dimensional Poincar\'e inequality} proposed in~\cite{Bolley_Gentil_Guillin_2018} permits us to obtain improved majorizations over the bounds obtained from the standard Poincar\'e inequality. However, the form of this improved majorization is implicit, which is difficult to optimize directly. Instead, we consider its use in obtaining improved certificates for the optimal features~$U_r^*$. Intriguingly, leveraging this improvement requires the use of \emph{lower bounds} on the approximation error, which presents significant computational challenges.  We leave addressing these issues, as well as potential applications of the dimensional Poincar\'e to more general ridge-based regression (see, e.g., \cite{bigoniNonlinearDimensionReduction2021,verdiere2023diffeomorphism}), to future work.

\subsection{Notations.} {Throughout we use $U_r \in \R^{d \times r}$ to denote matrices with orthonormal columns, so that $U_r^\top U_r = I_r$, where $I_r$ denotes the $r$ dimensional identity matrix. Often we write $U_\perp \in \mathbb{R}^{d \times (d-r)}$ as any orthonormal complement to $U_r$, which though is not unique, is sufficient for our purposes. For any $x \in \R^d$ we thus have the decomposition $x = U_r x_r + U_\perp x_\perp$, where $x_r \in \R^r$ and $x_\perp \in \R^{d-r}$; it will be clear from context that the subscript $\perp$ refers to a $d-r$ dimensional object.  We write $\lambda_k(A,B)$ to denote the $k$-th largest eigenvalue of the matrix pencil $(A,B)$, and $\lambda_k(A)$ to denote the $k$-th largest eigenvalue of matrix~$A$. The set $\R_+$ is used to denote non-negative real numbers.}

\section{Dimension Reduction with the KL Divergence}
\label{sec:KL}

\subsection{Certifiable bounds using the logarithmic Sobolev inequality}\label{sec:logSob}

We recall some properties of the optimal profile function from \cite[\S 2.1]{ZCLSM22} for the KL divergence. For any matrix $U_r \in \R^{d \times r}$ with $r \leq d$ orthonormal columns, and with $\widetilde\pi$ as in~\eqref{eq:piTildeIntro}, the minimizer of $\ell_r \mapsto \Dkl(\pi \,||\, \widetilde\pi) = \int\log(\frac{\d\pi}{\d\widetilde\pi})\d\pi$  is the conditional expectation
\begin{equation}
\label{eq:optKL}
\ell^{\KL}_r(\theta_r ) = \E_{X \sim \mu}\left[ \left.\frac{\d\pi}{\d\mu}(X) \,\right|\, U_r^\top X = \theta_r \right] .
\end{equation}
We denote by $\pioptKLr(U_r) = \pioptKLr(\cdot \mid U_r)$ the resulting probability measure given by
\begin{equation}
\label{eq:pioptKL}
 \d\pioptKLr(x \mid U_r) \propto \ell^{\KL}_r(U_r^\top x)\d\mu(x) \;.
\end{equation}
Let us emphasize that, by construction, the law of $U_r^\top \widetilde X$ and of $U_r^\top X$ are the same for any $U_r$, where $\widetilde X\sim\widetilde\pi^\KL_r(U_r)$ and $X\sim\pi$. In other words, the KL optimal profile \eqref{eq:optKL} ensures that the marginal laws of $\widetilde\pi^\KL_r(U_r)$ and $\pi$ are the same along the linear feature $U_r$.
Using the above notations, problem~\eqref{eq:loss} becomes
\begin{equation}\label{eq:lossKL}
    \min_{\substack{U_r \in \R^{d \times r} \\ U_r^\top U_r = I_r}} \Dkl(\pi \,||\,  \pioptKLr(U_r) ) .
\end{equation}
In the following we assume that $\mu=\mathcal{N}(0,I_d)$ is the standard normal measure.
In order to derive a majorant for $\Dkl(\pi \,||\,  \pioptKLr(U_r) )$ we employ the \emph{logarithmic Sobolev inequality} (LSI).

\begin{proposition}[Gaussian LSI]
\label{thm:LSI}
The entropy $\Ent_\mu(f) = \int f \ln f \d\mu - \int f \d\mu \ln \int f \d \mu$ of any smooth positive function $f: \R^d \to \R_+$ with respect to the standard Gaussian~$\mu=\mathcal{N}(0,I_d)$ satisfies
\begin{equation}
\label{eq:LSI}
\frac{1}{2 \int f \d\mu} \, \left\| \int \nabla f \d\mu \right\|_2^2 \leq \Ent_\mu(f) \leq \frac{1}{2} \int \|\nabla \ln f\|_2^2 f \d\mu.
\end{equation}
We refer to the upper bound as the \emph{logarithmic Sobolev inequality (LSI)}, and to the lower bound as the \emph{reverse logarithmic Sobolev inequality (reverse LSI)}.
\end{proposition}

{The logarithmic Sobolev inequality was independently proposed by Stam~\cite{stamInequalitiesSatisfiedQuantities1959a} and Gross~\cite{Gross_1975}, though we are unaware of the seminal reference for the reverse inequality.} Both bounds are formulated for general Markov diffusion operators in \cite[Ch. 5]{Bakry2014}, which we apply in Appendix~\ref{sec:proofLSI} to obtain the above proposition.
Note the choice of $f = \d\pi / \d\mu$ in \eqref{eq:LSI} permits us to bound the divergence $\Dkl(\pi \,||\,\mu)$ following
\[
\frac{1}{2} \left\| \E_\pi\left[ \nabla \ln \left(\frac{\d\pi}{\d\mu}\right) \right]\right\|_2^2
\leq \Dkl(\pi \,||\,\mu)
\leq \frac{1}{2}\E_\pi\left[ \left\| \nabla \ln \left(\frac{\d\pi}{\d\mu}\right) \right\|_2^2 \right] \;.
\]
The next theorem further motivates the connection between dimension reduction for the KL divergence and the logarithmic Sobolev inequality, and we leave the proof to Appendix~\ref{sec:lowerCDRproof}. While this strategy was originally proposed in~\cite[Corr. 2.10]{ZCLSM22}, the formulation of the lower bound is a new contribution.

\begin{theorem}
\label{thm:CDR}
Consider the probability measure $\d\pi(x) \propto \ell(x)\d\mu(x)$ for some smooth $\ell:\R^d \to \R_+$ and standard Gaussian~$\mu=\mathcal{N}(0,I_d)$. Then, for any matrix~$U_r \in \R^{d \times r}$ with $r \leq d$ orthonormal columns, the measure $\pioptKLr(U_r)$, as in \eqref{eq:pioptKL}, satisfies
\begin{equation}
\label{eq:CDRbound}
\frac{1}{2}\| (I_d - U_rU_r^\top) m(\pi) \|_2^2  \leq
\Dkl(\pi \,||\, \pioptKLr(U_r)) \leq
\frac{1}{2}  \trace\lb  (I_d - U_rU_r^\top) H(\pi||\mu) \rb ,
\end{equation}
where $m(\pi)=\E_\pi[X]$ is the mean of $\pi$ and $H(\pi||\mu) = \E_{X \sim \pi}[\nabla \ln \ell(X) \nabla \ln \ell(X)^\top]$ is the Fisher information matrix of $\pi$ relative to $\mu$.
\end{theorem}

Minimizing the right-hand side of~\eqref{eq:CDRbound} corresponds to maximizing $U_r\mapsto\trace(U_r^\top H(\pi||\mu) U_r)$ and yields
the matrix~$U_r^*$ whose columns contain the $r$~leading eigenvectors of~$H(\pi||\mu)$, see \cite{Fan_1949}.
The approximation of $\pi$ by $\pioptKLr(U_r^*)$ is therefore furnished with the error certificates
\begin{equation}
\label{eq:CDRerror}
\Dkl(\pi \,||\, \pioptKLr(U_r^*))
\leq \frac{1}{2}\sum_{k=r+1}^d \lambda_k(H(\pi||\mu)),
\end{equation}
where $\lambda_k(A)$ denotes the $k$-th largest eigenvalue of~$A$.  While $U_r^*$ may not be the global minimizer of~\eqref{eq:lossKL}, this result \emph{certifies} that
its approximation error is no larger than $\frac{1}{2}\sum_{k=r+1}^d \lambda_k(H(\pi||\mu))$, which can be made arbitrarily small as $r \to d$.
When the spectrum of $H$ decays rapidly, it suffices to choose
rank~$r \ll d$ to obtain an accurate approximation.
The lower bound~\eqref{eq:CDRbound} provides a guarantee of the smallest achievable approximation error, which may also inform the practitioner's choice of~$r$.

\subsection{Improved bounds using dimensional inequalities}\label{sec:dimlogSob}

We now introduce a functional inequality which is uniformly sharper than the LSI for Gaussian measures.
We defer its proof to Appendix~\ref{sec:fun_inequalities}. In the following, we employ the notation $v^{\otimes2} = vv^T$.

\begin{proposition}[Dimensional Gaussian LSI]
\label{thm:dimLSI}
For standard Gaussian~$\mu=\mathcal{N}(0,I_d)$ on $\R^d$, and for any smooth function $f : \R^d \to \R_+$ such that $\int f \d\mu = 1$, we have the \emph{dimensional logarithmic Sobolev} inequality and the \emph{reverse dimensional logarithmic Sobolev} inequality
\begin{align}
\Ent_\mu \lb f \rb
&\leq {\frac{1}{2} \int \|x\|_2^2 f \d\mu - \frac{d}{2}
+\frac{1}{2}\ln \det\lb \int \lb \nabla\ln f -x \rb^{\otimes2}f \d\mu \rb,
}\label{eq:dimLSI} \\
\Ent_\mu(f) &\geq \frac{1}{2}\int \|x\|_2^2 f \d\mu - \frac{d}{2} - \frac{1}{2} \ln\det \lb \int x^{\otimes 2} \,f \d\mu - \lb \int x f \d\mu \rb^{\otimes 2} \rb.\label{eq:revdimLSI}
\end{align}
\end{proposition}

{The Gaussian dimensional LSI was discovered independently by  Dembo~\cite{Dembo_1990} and Bakry and Ledoux~\cite[Prop. 2]{Bakry_Ledoux_2006}, though the former provided a stronger proof technique which results in Proposition~\ref{thm:dimLSI}; see Appendix~\ref{sec:fun_inequalities} for further discussions. Equations~\eqref{eq:dimLSI} and~\eqref{eq:revdimLSI} can also be obtained using the more general `local intrinsic dimensional logarithmic Sobolev inequalities' given by equations (29) and (30) in~\cite{Eskenazis_Shenfeld_2023} by setting $T=1$ and $x=0$.}
{Logarithmic Sobolev inequalities in the spirit of Proposition~\ref{thm:dimLSI} can also be found in \cite[Section 2.2]{Eskenazis_Shenfeld_2023} for general measures with homogeneous densities beyond the standard Gaussian.}

\begin{remark}
We are only aware of the following measures satisfying the dimensional LSI in the Euclidean geometry: Gaussian measures, push-forwards of Gaussian measures by functions with Lipschitz gradients, and perturbations of Gaussians measures in the sense of Holley-Stroock~\cite{holleyLogarithmicSobolevInequalities1987}.
Formulations of the dimensional logarithmic Sobolev inequalities exist for measures on other manifolds, but we refer the interested reader to, e.g., \cite[Thm 6.7.3]{Bakry2014} for further details.
\end{remark}

\begin{remark}
{The equality cases of the original LSI from Proposition~\ref{thm:LSI} are given by translations of Gaussian with identity covariance~\cite[\S2]{Bakry_Ledoux_2006}. However, Proposition~\ref{thm:dimLSI} is sharp for all Gaussians.}
\end{remark}

{The dimensional logarithmic Sobolev inequality~\eqref{eq:dimLSI} uniformly improves on the logarithmic Sobolev inequality in Proposition~\ref{thm:LSI} for all test functions~$f$ satisfying the assumptions of the theorem. This can be seen by applying the log-determinant upper bound $\ln\det(S) \leq \trace S - d$ for positive semi-definite $S \in \R^{d \times d}$. Specifically, choosing $S = \int (\nabla \ln f - x)^{\otimes 2} f \d\mu$ shows that
\[
\text{r.h.s.~of}~\eqref{eq:dimLSI} \leq \frac{1}{2}\int \|\nabla \ln f\|_2^2 \d\mu + \int \|x\|_2^2 f \d\mu - d - \int \langle \nabla f, x \rangle \d \mu \leq \frac{1}{2}\int \|\nabla \ln f\|_2^2 \d\mu  \;,
\]
where the second inequality stems from noting $\int f d \mu = 1$  and applying Stein's identity~\cite[Lemma 1]{steinEstimationMeanMultivariate1981} to obtain
\[
\int \langle \nabla f, x \rangle \d \mu = \sum_{i=1}^d \left( \int x_i^2 f \d\mu - \int f \d\mu \right) = \int \|x\|_2^2 f \d\mu - d \;.
\]
A similar exercise, except choosing $S = \int x^{\otimes 2} f \d\mu - (\int x f \d\mu)^{\otimes 2}$ instead, shows that the reverse dimensional logarithmic Sobolev inequality~\eqref{eq:revdimLSI} similarly improves on the reverse LSI in Proposition~\ref{thm:LSI} uniformly.}

These improvements allow us to derive improved majorants for the KL divergence between~$\pi$ and its approximation~$\pioptKLr(U_r)$, as shown by the following theorem, whose proof we leave to Appendix~\ref{sec:dimCDRproof}.

\begin{theorem}
\label{thm:dimCDR}
Consider the standard Gaussian measure~$\mu=\mathcal{N}(0,I_d)$ on $\R^d$. Then, for any matrix~$U_r \in \R^{d \times r}$ with $r \leq d$ orthonormal columns, the measure $\pioptKLr(U_r)$, as in \eqref{eq:pioptKL}, satisfies
\begin{equation}
\label{eq:dimCDRbound}
 \mcJ_\KL^\uparrow(U_r) \leq D_\KL(\pi \,||\, \pioptKLr(U_r)) \leq \mcJ_\KL^\downarrow(U_r),
\end{equation}
where
\begin{align}
&\mcJ_\KL^\downarrow(U_r) =  \frac{\trace( M(\pi)  ) -d + \ln \det \lb H(\pi) \rb}{2} - \frac{\trace( U_r^\top M(\pi) U_r )-r + \ln \det (U_r^\top H(\pi)^{-1} U_r )}{2} ,  \label{eq:dimCDR_up} \\
&\mcJ_\KL^\uparrow(U_r) =   \frac{\trace(M(\pi))-d- \ln\det(C(\pi))}{2} - \frac{\trace(U_r^\top M(\pi) U_r)-r -\ln\det (U_r^\top C(\pi)^{-1} U_r)}{2} .
\label{eq:dimCDR_down}
\end{align}
Here $H(\pi) = \E_{\pi}[\nabla\ln(\frac{\d\pi}{\d x})\nabla\ln(\frac{\d\pi}{\d x})^\top]$ denotes the Fisher information matrix of~$\pi$ and $M(\pi) = \E_{\pi}[ XX^\top]$ is the second moment matrix of~$\pi$, and $ C(\pi) = M(\pi) - \E_{\pi}[ X]\E_{\pi}[X]^\top $ is the covariance matrix of~$\pi$.
\end{theorem}
\begin{remark}
Let $U_\perp$ denote any orthonormal completion to $U_r$ such that $[ U_r , \;  U_\perp ]\in \mathbb{R}^{d \times d}$ is unitary. Using standard matrix determinant identities detailed in Appendix~\ref{sec:matrixdeterminant}, we have the equivalent representations
\begin{align}
\mcJ_\KL^\downarrow(U_r) &= \frac{ \trace( U_\perp^\top M(\pi) U_\perp ) -(d-r) + \ln\det ( U_\perp^\top H(\pi) U_\perp ) }{2} \label{eq:dimCDRperp_up} \\
\mcJ_\KL^\uparrow(U_r) &=  \frac{ \trace( U_\perp^\top M(\pi) U_\perp ) -(d-r) -  \ln\det ( U_\perp^\top C(\pi) U_\perp) }{2},
\label{eq:dimCDRperp_down}
\end{align}
which do not require computing matrix inverses to evaluate.
The above expressions resemble the \emph{log-determinant divergence}, or \emph{Burg divergence}, $D_\text{burg}(X,Y) = \trace(XY^{-1})-d-\ln\det(XY^{-1})$ between two symmetric positive definite matrices $X,Y\in\R^{d\times d}$, see \cite{kulisLearningLowrankKernel2006}, and it would be interesting to establish whether there is a deeper connection between these ideas.
\end{remark}

The inequalities in Theorem \ref{thm:dimCDR} improve on those in Theorem \ref{thm:CDR}
just as the dimensional LSI improves on the LSI,
meaning that uniformly over all orthonormal matrices~$U_r \in \R^{d \times r}$ we have
$$
 \frac{1}{2}\| (I_d - U_rU_r^\top) m(\pi) \|_2^2 \leq \mcJ_\KL^\uparrow(U_r)
 \qquad\text{and}\qquad
 \mcJ_\KL^\downarrow(U_r) \leq \frac{1}{2}  \trace\lb  (I_d - U_rU_r^\top) H(\pi||\mu) \rb \;.
$$
{For the majorant, this can be demonstrated by applying the log-determinant upper bound $\ln \det(S)  \leq \trace(S) - (d-r)$ with~$S = U_\perp^\top H(\pi) U_\perp \in \R^{(d-r)\times(d-r)}$ to  \eqref{eq:dimCDRperp_up}, showing that
\[
\text{r.h.s. of}~\eqref{eq:dimCDRperp_up} \leq \frac{1}{2} \trace(U_\perp^\top M(\pi) U_\perp) - (d-r) + \frac{1}{2}\trace(U_\perp^\top H(\pi) U_\perp) \;.
\]
We then note that the relative Fisher information matrix $H(\pi||\mu)$ with $\mu=\mathcal{N}(0,I_d)$ decomposes as
\begin{align}
 H(\pi||\mu) &\overset{\eqref{eq:defH}}{=} \E_\pi\left[ \left(\nabla\ln\left(\frac{\d\pi}{\d x}\right) + x\right)\left(\nabla\ln\left(\frac{\d\pi}{\d x}\right) + x\right)^\top \right] \nonumber\\
 &\overset{\eqref{eq:defHandM}}{=}H(\pi) +\left(\int \nabla\ln\left(\frac{\d\pi}{\d x}\right)  x^\top +  x\nabla\ln\left(\frac{\d\pi}{\d x}\right)^\top \d\pi \right) +M(\pi) \nonumber\\
 &=H(\pi) -2 I_d + M(\pi) \;. \label{eq:FIMidentity}
\end{align}
The last step follows from integration by parts since, writing $p(x) = \tfrac{\d\pi}{\d x}$, we have
\[
\int \frac{\partial \ln p(x)}{\partial x_i} x_j \, p(x) \d x = \int x_j \frac{\partial}{\partial x_i}p(x) \d x = -\int p(x) \frac{\partial}{\partial x_i} x_j \; \d x = -\delta_{ij}
\]
for any $1 \leq i,j \leq d$, where in the last equality we make use of the fact that $\pi$ is a probability measure. Substituting this identity for~$H(\pi\,||\,\mu)$ above then recovers the desired inequality. The improvement of the minorant~\eqref{eq:dimCDRperp_down} can be similarly shown by choosing $S = U_\perp^\top C(\pi) U_\perp$.
}

Theorem~\ref{thm:dimCDR} sharpens the error certificates from Theorem~\ref{thm:CDR}, possibly by an \emph{exponential} factor as seen from the log-determinant term.
Minimizing $\mcJ_\KL^\downarrow(\cdot)$ thus results in features with tighter certifiable approximation guarantees compared to minimizing the majorant in \eqref{eq:CDRbound}, which we show in Appendix~\ref{sec:firstorder}.
The following proposition also shows that the majorization~\eqref{eq:dimCDR_up} monotonically decreases over sequences of nested features $\textrm{span}(U_r) \subseteq \textrm{span}(U_{r+1})$. This implies $\min_{U_{r+1}} \mcJ_\KL^\downarrow(U_{r+1}) \leq \min_{U_r} \mcJ_\KL^\downarrow(U_r)$, and thus one always obtains features with improved error certificates when increasing the feature rank~$r$, as intuitively desired. We defer the proof to Appendix~\ref{sec:proof_dim_monotone}.
However, unlike the minimizers $\{U_1^*,U_2^*,\hdots\}$ of the LSI bound \eqref{eq:CDRbound} which satisfy $\textrm{span}(U_r^*) \subseteq \textrm{span}(U_{r+1}^*)$, the minimizers of $\mcJ_\KL^\downarrow(\cdot)$ are not necessarily nested.

\begin{proposition}
\label{thm:dim_monotone}
Let $\{u_k\}_{k=1}^d$ denote any orthonormal basis of $\R^d$ and consider the sequence $\{U_r = (u_1, \ldots, u_r) \in \R^{d \times r}\}_{1 \leq r \leq d}$.
Then, we have
$\mcJ_\KL^\downarrow(U_{r+1}) \leq \mcJ_\KL^\downarrow(U_r)$ for all $r<d$.
As a consequence, for all $r<d$ we have
\begin{equation}\label{eq:dim_monotone}
 \min_{\substack{U_{r+1} \in \R^{d \times (r+1)} \\ U_{r+1}^\top U_{r+1} = I_{r+1}}} \mcJ_\KL^\downarrow(U_{r+1})
 \leq
 \min_{\substack{U_r \in \R^{d \times r} \\ U_r^\top U_r = I_r}} \mcJ_\KL^\downarrow(U_r) .
\end{equation}

\end{proposition}

There are additional numerical considerations to minimizing $\mcJ_\KL^\downarrow(\cdot)$, however.
When viewed as a function over Euclidean vectors in~$\R^{d \times r}$, this objective is non-convex due to the ortho-normality constraint $U_r^\top U_r = I_r$.
Alternatively, when considered as optimization over
the Grassmann manifold~$\mathrm{Gr}(d,r)$ --- the manifold of $r$ dimensional subspaces equipped with the canonical Euclidean metric~\cite{edelmanGeometryAlgorithmsOrthogonality1998} --- it is also impossible for this function to be globally geodesically convex, as such smooth functions cannot exist on compact manifolds~\cite[Corr 11.10]{Boumal_2023}.
Both observations suggest that (Riemannian) gradient descent may stagnate at local minima.
Despite this, our numerical experiments in Pymanopt~\cite{pymanopt} --- with structured $M(\pi)$ and $H(\pi)$ arising from statistical problems, or more generally with random $M(\pi) \succ 0$ and $H(\pi) \succ 0$ --- suggest that Riemannian gradient descent~\cite{edelmanGeometryAlgorithmsOrthogonality1998} converges to the same solution for all initial conditions. This leads us to the following.

\begin{conjecture}
\label{conjecture:benign}

The function $\bar{\mcJ}_\KL^\downarrow:\mathrm{Gr}(d,r)\rightarrow\R$ defined by $\bar{\mcJ}_\KL^\downarrow( \mathrm{range}\{U_r\} ) = \mcJ_\KL^\downarrow(U_r) $ as in \eqref{eq:dimCDR_up} for all orthonormal matrix~$U_r \in \R^{d \times r}$, is benignly non-convex on the Grassmann manifold $\mathrm{Gr}(d,r)$.
\end{conjecture}

Benign non-convexity refers to non-convex functions for which every local minimum is a global minimum; see, e.g., \cite{sunWhenAreNonconvex2016,benignnonconvexitylist} for a curated list of examples. Geometrically, all saddle points, if any, must therefore have at least one escape direction, providing a mechanism for first-order descent methods to reach a global minimizer almost surely.
While there has been significant interest in the optimization community on discovering such functions, to our knowledge the proof techniques for demonstrating benign non-convexity on the Grassmannian require knowing the global minima in closed form; see, e.g., \cite{ahnRiemannianPerspectiveMatrix2021, alimisisGeodesicConvexitySymmetric2023}. We are unable to prove such a characterization, and we can only derive necessary and sufficient conditions for first order critical points in Appendix~\ref{sec:firstorder}.

We now revisit the central question concerning the choice of functional inequality for dimension reduction.
Evidently, the use of the dimensional LSI improves upon the LSI: Theorem~\ref{thm:dimCDR} shows that sharper error certificates are obtained which
are also computationally straightforward to minimize, assuming Conjecture~\ref{conjecture:benign} holds.
In fact, the following proposition suggests that this is, in some sense, the \emph{optimal} choice of functional inequality for dimension reduction under the KL divergence. Specifically, minimizing the majorant in~\eqref{eq:dimCDR_up} is equivalent to minimizing the KL divergence when the target is Gaussian, even though the dimensional LSI was not originally derived for the purpose of dimension reduction for Gaussian measures. Accordingly, Theorem~\ref{thm:dimCDR} can be interpreted as the extension of dimension reduction towards non-Gaussian distributions.

\begin{proposition}
\label{thm:dimCDRexactgauss}
For any non-degenerate Gaussian measure~$\pi$, we have the equality $\mcJ_\KL^\downarrow(U_r) = \Dkl(\pi \,||\, \pioptKLr(U_r))$ for all matrices $U_r \in \R^{d \times r}$ with orthonormal columns.
\end{proposition}
\begin{proof}
~Let $\pi=\mathcal{N}(m,C)$ be a Gaussian measure with mean~$m$ and covariance~$C\succ 0$. For any orthonormal matrix $U = [U_r, \,\, U_\perp]$ the measure $\pioptKLr(U_r)$ defined as in~\eqref{eq:pioptKL} is also Gaussian, but with mean~$U_rU_r^\top m$ and covariance
$U \mathrm{diag}( U_r^\top C U_r , \, I_{d-r}) U^\top$. {Recall the KL divergence between two Gaussian measures is given by
\[
\Dkl(\mathcal{N}(\mu_1, \Sigma_1) \,||\, \mathcal{N}(\mu_2, \Sigma_2)) =  \frac{1}{2} \left(\ln\det \Sigma_2 - \ln\det\Sigma_1 -d + \trace(\Sigma_2^{-1} \Sigma_1) +  (\mu_2-\mu_1)^\top \Sigma_2^{-1} (\mu_2-\mu_1) \right) \;.
\]
Noting
\[
\trace \left(\begin{pmatrix} (U_r^\top C U_r)^{-1} & \\ & I_{d-r}\end{pmatrix} U^\top C U \right) = r + \trace( U_\perp^\top C U_\perp) \;,
\]
direct computation then shows that
\[
\Dkl(\pi\,||\,\widetilde\pi_r) =  \frac{1}{2}( \ln\det(U_r^\top C U_r) - \ln\det C -(d-r) + \trace U_\perp^\top (C + mm^\top) U_\perp).
\]
Combining this with the observation $H(\pi) = \mathbb{E}_\pi[C^{-1}(X-m)(X-m)^\top C^{-1}] = C^{-1}$ and the standard identity $C = \E_{X\sim \pi}[XX^\top] - mm^\top$ demonstrates the equivalence to~\eqref{eq:dimCDR_up}.
}
\end{proof}

\subsubsection{Numerical example: Linear Gaussian inverse problems}

We consider linear Gaussian inverse problems in this section, and leave applications to a broader class of Bayesian inverse problems to~\S\ref{sec:bayesian}. We choose the reference measure~$\mu$ to be the prior distribution of the model parameters~$X$, assumed to be standard Gaussian in $\R^{d}$, and we assume data generated by the model $Y = AX + \epsilon$, where $A \in \R^{n_y \times d}$ and $\epsilon \sim \mcN(0, I_d)$. The posterior measure of $X|Y=y$ is our target measure~$\pi$, given by a Gaussian with mean~$m(\pi) = (I_{d} + A^\top A)^{-1}A^\top y$ and covariance~$C(\pi) = (I_{d} + A^\top A)^{-1}$. For any feature $U_r \in \R^{d \times r}$, the approximation $\pioptKLr(U_r)$ is also Gaussian with mean $U_rU_r^\top m(\pi)$ and covariance~$I_{d} + U_rU_r^\top(C(\pi) - I_{d})U_rU_r^\top$.

{Under the assumption $\mu_{X \mid Y} = \E_{\pioptKLr}[X]$, or equivalently, $(I_d - U_rU_r^\top) m(\pi) = 0$, Cui et al.~\cite[Lemma 2.2]{Cui_Martin_Marzouk_Solonen_Spantini_2014} show that the $r$ leading eigenvectors of the relative Fisher information matrix~$H(\pi\,||\,\mu)$ are globally optimal for~\eqref{eq:loss} with the KL divergence. However, note that the matching mean condition, which depends on the observed data~$Y$, almost surely cannot be satisfied. Hence, in general these features are not provably optimal for linear Gaussian inverse problems.}

Instead, we claim that globally optimal features can be determined by minimizing $U_r \mapsto \mcJ^{\downarrow}(U_r)$, assuming the benign convexity conjecture. We evaluate the stability of this objective when the Fisher information matrix and the second moment matrix are empirically estimated from samples. While we use standard Monte Carlo for $\widehat{H}(\pi || \mu)$ and $\widehat{M}(\pi)$, we use the estimator $\widehat H(\pi) := \widehat H(\pi || \mu) - \widehat M(\pi) + 2I_d$ following~\eqref{eq:FIMidentity} rather than the Monte Carlo estimate of $\E_{\pi}[\nabla \ln \pi \; \nabla \ln \pi^\top]$ as we observed this produced more stable numerical results.

Figure~\ref{fig:lingaussKL} compares the performance of the estimator for a linear Gaussian inverse problem {with $d=50$ unknowns, and $n_y = 50$ observations}. The solid black curve depicts the result of minimizing~\eqref{eq:dimCDR_up} with analytical~$H(\pi)$ and $M(\pi)$, whereas the solid gray curve depicts the loss corresponding to~\eqref{eq:CDRerror}. The dashed lines show the result of minimizing~\eqref{eq:dimCDR_up} using empirically estimated matrices of sample size~$n_s$ --- observe that they underestimate the true error due to overfitting. Instead, when evaluated on a `testing' set of $\widehat{H}(\pi)$ and $\widehat M(\pi)$ computed with $200$ samples, the solid coloured curves show that the estimated $\widehat U_r^\downarrow$ are close to the global minimizer.

\begin{figure}[!h]
\centering\includegraphics[width=\textwidth]{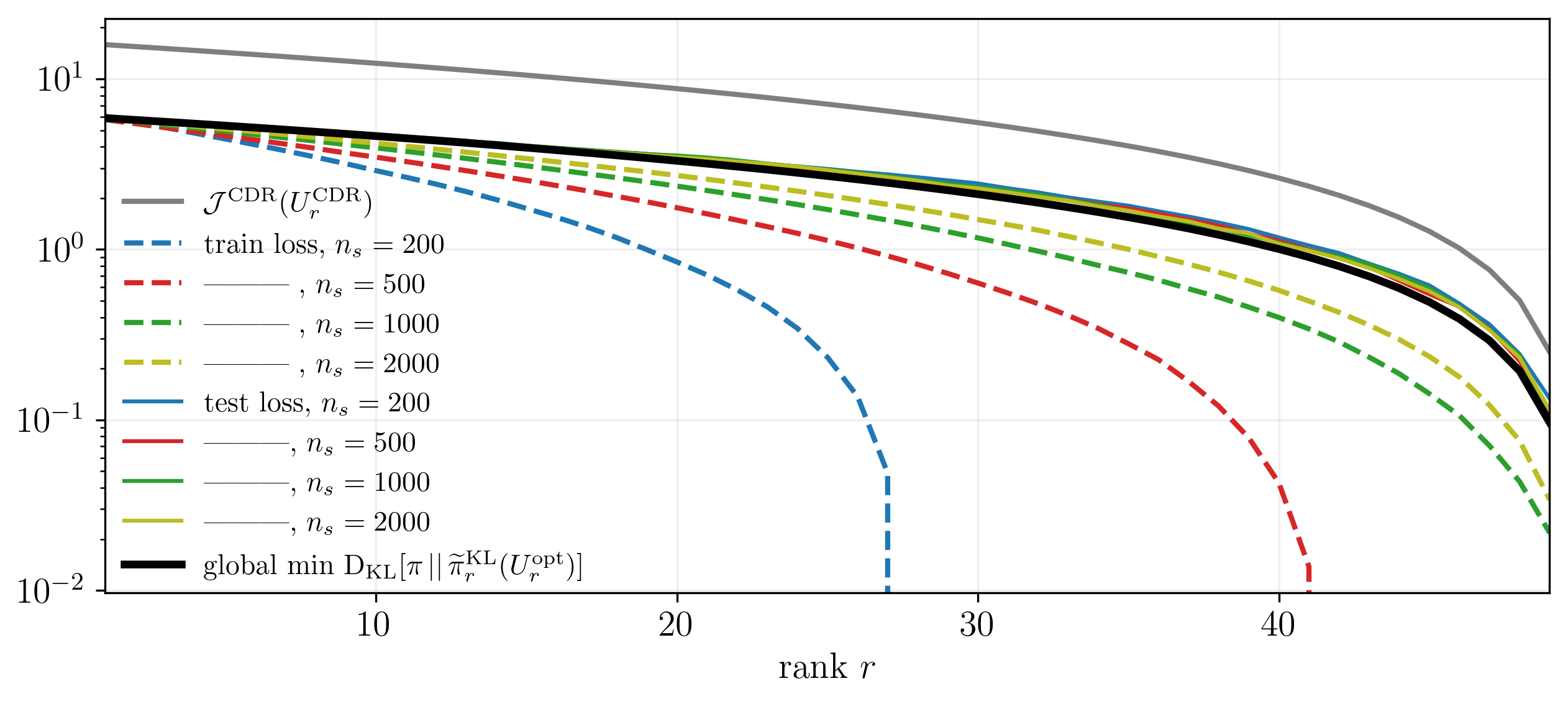}
\caption{Comparison of approximation error certificates for a Gaussian target measure of size $d = n_y = 50$. The globally optimal features~$U_r^\opt$ are obtained by minimizing~\eqref{eq:dimCDR_up}, and their achieved KL error is shown in black. The certificate obtained from the LSI majorization~\eqref{eq:CDRbound} is shown in gray. The minimizers of the empirically estimated dimensional LSI majorant~\eqref{eq:dimCDR_up} is shown by the coloured curves, where solid lines correspond to testing error, and dashed lines correspond to training error.}
\label{fig:lingaussKL}
\end{figure}

\subsubsection{Numerical example: Rosenbrock}
\label{sec:rosenbrock1}

Consider the two-dimensional Rosenbrock distribution $\pi = T_\#\mu$ with push-forward $T(z_1, z_2) = [z_1-0.5,~(z_1-0.5)^2 + \sqrt{0.2}z_2]$ and standard Gaussian~$\mu$. Figure~\ref{fig:bananapiopt} shows the probability density of~$\pi$, alongside its approximations~$\pioptKLr(U_1^*)$ with feature $U_1^*$, obtained by minimizing the majorant~\eqref{eq:CDRbound}, and~$\pioptKLr(U_1^\downarrow)$ with feature $U_1^\downarrow$ obtained by minimizing the majorant~\eqref{eq:dimCDR_up}.
The computation of each approximation via~\eqref{eq:pioptKL} uses Gauss-Hermite quadratures.
The relative Fisher information matrix~$H(\pi||\mu)$ and moment matrix~$M(\pi)$ are similarly computed with quadrature, yielding the Fisher information matrix $H(\pi)$ via~\eqref{eq:FIMidentity}.

\begin{figure}[!h]
\centering\includegraphics[width=0.95\textwidth]{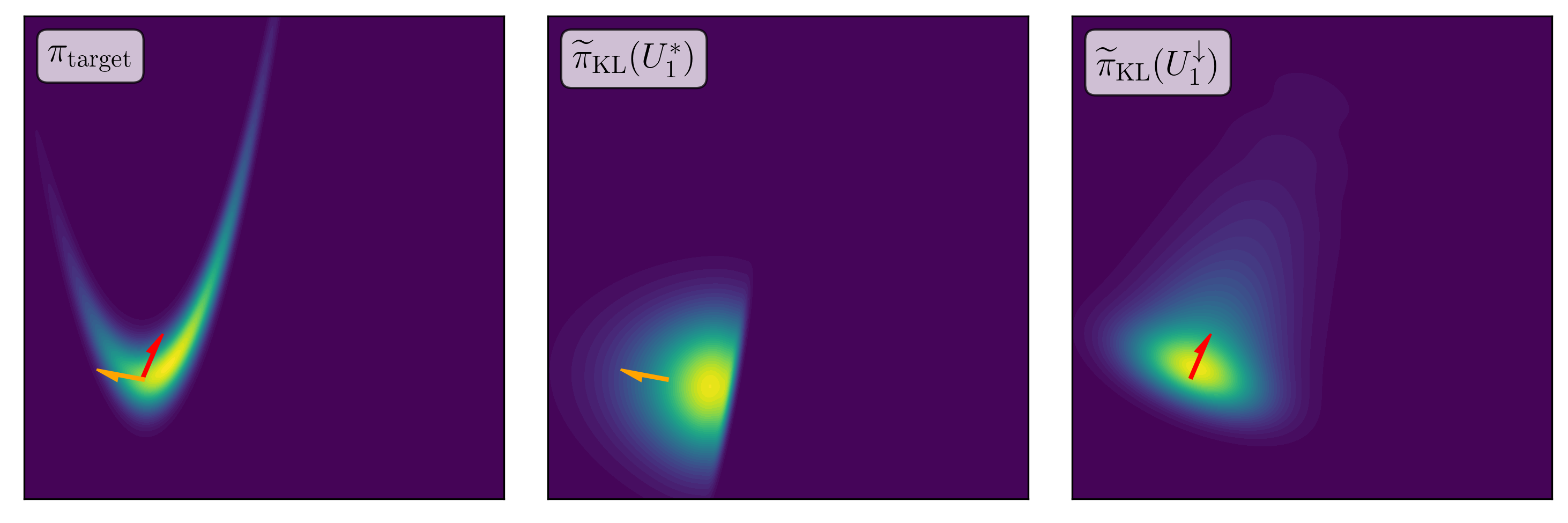}
\caption{ (Left) Probability density of target measure. (Centre) Approximate measure using feature~$U_1^*$ that minimizes the majorant in~\eqref{eq:CDRbound}. (Right) Approximate measure using feature~$U_1^\downarrow$ that minimizes the majorant in ~\eqref{eq:dimCDRbound}.}
\label{fig:bananapiopt}
\end{figure}

We parameterize features $U_1=(\cos(\theta),\sin(\theta))\in\R^2$ by a unique angle $\theta \in [0, \pi)$. For each $\theta$, we compute $\Dkl(\pi||\pioptKLr(U_1(\theta)))$ using numerical quadratures, resulting in the solid blue curve in Figure~\ref{fig:bananaKL-loss}. The shaded gray region corresponds to the majorization and minorization provided in Theorem~\ref{thm:CDR}. Minimizing the majorant yields $\theta(U_1^*)$, depicted by the dashed yellow line. In contrast, the shaded blue region corresponds to the majorization and minorization from Theorem~\ref{thm:dimCDR}. The minimizer $\theta(U_1^\downarrow)$ of~\eqref{eq:dimCDR_up} is shown by the dashed red line.

We observe that the bounds corresponding to the dimensional LSI are significantly tighter than those corresponding to the non-dimensional LSI, confirming our analysis. {Moreover, though the upper blue envelope corresponding to $\mcJ_\KL^\downarrow(\cdot)$ is non-convex, it exhibits only a single global minimum, consistent with the benign non-convexity conjecture stated in Conjecture \ref{conjecture:benign}}. Surprisingly, we also observe that the dimensional LSI minimizer~$\theta(U_1^\downarrow)$ coincides with the global minimizer of the KL divergence, {whereas the LSI minimizer $\theta(U_1^*)$ achieves nearly the worst possible error with respect to the true divergence shown by the solid blue curve.}

\begin{figure}[!h]
\centering\includegraphics[width=0.8\textwidth]{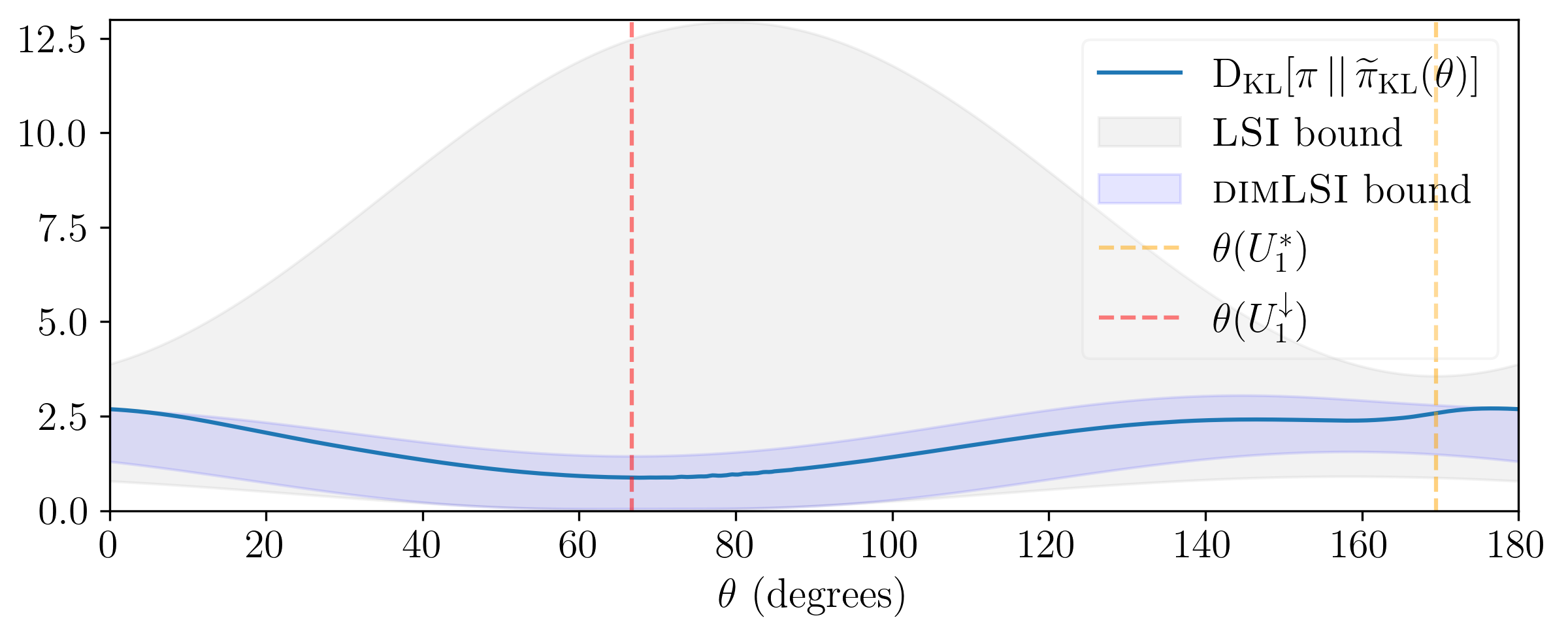}
\caption{(Rosenbrock) KL divergence $\Dkl(\pi|| \pioptKLr(U_1(\theta)))$ with $U_1(\theta)=(\cos(\theta),\sin(\theta))$ for $\theta\in[0^\circ,180^\circ)$, shown by the solid blue line, and its bounds obtained either by the LSI of Theorem \ref{thm:CDR}, depicted by the gray shaded region, or by the dimensional LSI of Theorem \ref{thm:dimCDR}, depicted by the blue shaded region.}
\label{fig:bananaKL-loss}
\end{figure}

\subsection{Choosing the reference measure $\mu$ as the best Gaussian approximation to $\pi$}

As motivation, note that for an \emph{isotropic} measure $\pi$, so that $M(\pi)=I_d$, the upper bound in \eqref{eq:dimCDRbound} of Theorem~\ref{thm:dimCDR} simplifies as
\begin{align*}
 \mcJ_\KL^\downarrow(U_r)
 &= \frac{1}{2}\ln\det (  H(\pi)  ) - \frac{1}{2}\ln\det ( U_r^\top H(\pi) U_r ) \;.
\end{align*}
If in addition $\pi$ is centered so that $m(\pi)=0$, the lower bound vanishes, i.e., $\mcJ_\KL^\uparrow(U_r)=0$ for any $U_r$ with orthonormal columns.
In this setting,  the matrix $U_r^\downarrow$ with columns containing the $r$ leading eigenvectors of the Fisher information matrix $H(\pi)$ is the global minimizer of $\mcJ_\KL^\downarrow(\cdot)$ \cite{Fan_1949}, and evaluating
the upper bound \eqref{eq:dimCDRbound} achieves
$$
\Dkl(\pi \,||\, \pioptKLr(U_r^\downarrow))
\leq \frac{1}{2}\sum_{k=r+1}^d \ln (  \lambda_k(H(\pi)) ) ,
$$
where $\lambda_k(A)$ is the $k$-th largest eigenvalue of $A$.
Recalling \eqref{eq:FIMidentity}, we have $H(\pi)=H(\pi||\mu)+I_d$ so that $U_r^\downarrow=U_r^*$, where $U_r^*$ is the minimizer of the LSI majorant \eqref{eq:CDRbound}. This further implies that we have $\lambda_k(H(\pi))= 1+\lambda_k(H(\pi||\mu))$, and using the inequality $\ln(1+t)\leq t$ demonstrates that the above sharpens \eqref{eq:CDRerror}.

The same reasoning can be generalized to non-isotropic/non-centered measures $\pi$ by changing the reference measure $\mu$ to be the Gaussian measure $\mu=\mathcal{N}(m(\pi),C(\pi))$. This choice is natural since it is the best Gaussian approximation to $\pi$, in the sense of minimizing $\mu\mapsto\Dkl(\pi||\mu)$ over the set of Gaussian measures.
The proof is given in Appendix \ref{proof:dimCDR_tilted}.

\begin{theorem}\label{thm:dimCDR_tilted}
 Let $\pi$ be a probability measure on $\R^d$ with mean $m(\pi)$ and non-singular covariance matrix $C(\pi)$. Consider the Gaussian measure $\mu=\mathcal{N}(m(\pi),C(\pi))$ with the same mean and covariance. Then, for any $U_r\in\R^{d\times r}$ such that $U_r^\top C(\pi) U_r = I_r$, the approximation $\pioptKLr(U_r)$ as in \eqref{eq:pioptKL} satisfies
 \begin{equation}\label{eq:dimCDR_tilted}
  \Dkl(\pi \,||\, \pioptKLr(U_r))
  \leq \frac{1}{2} \ln\det(C(\pi)H(\pi) ) - \frac{1}{2} \ln\det( V_r^\top H(\pi) V_r)  ,
 \end{equation}
 where $V_r = C(\pi)U_r$ and where $H(\pi)=\E_\pi[\nabla\ln(\frac{\d\pi}{\d x})\nabla\ln(\frac{\d\pi}{\d x})^\top]$ is the Fisher information matrix of $\pi$.
 In particular, the minimizer of the above right-hand side is given by $U_r^\downarrow = C(\pi)^{-1} V_r^\downarrow$ where $V_r^\downarrow = [v_1,\hdots,v_r]$ contains the $r$ largest generalized eigenvectors of the matrix pair $(H(\pi),C(\pi)^{-1})$, meaning
 $
  H(\pi) v_k = \lambda_k(H(\pi), C(\pi)^{-1}) C(\pi)^{-1} v_k ,
 $
 with the ortho-normality constraint $v_i^\top C(\pi)^{-1} v_j = \delta_{i,j}$.
 In this case, we obtain
 \begin{equation}\label{eq:dimCDR_tilted_optimal}
 \Dkl(\pi || \pioptKLr(U_r^\downarrow))
 \leq \frac{1}{2}\sum_{k=r+1}^d \ln (  { \lambda_k(H(\pi), C(\pi)^{-1})}  ).
 \end{equation}
\end{theorem}

{In the same spirit as selecting the features~$U_r$ which minimize majorizations of the KL divergence in~\cite{ZCLSM22}}, Theorem \ref{thm:dimCDR_tilted} suggests that the deviations between~$\pi$ and its best Gaussian approximation~$\mu=\mathcal{N}(m(\pi),C(\pi))$ are controlled by the relative deviations --- in the sense of generalized Rayleigh quotients --- of the Fisher information matrix $H(\pi)$ and the precision matrix $C(\pi)^{-1}$.
Note that for Gaussian target measures $\pi=\mathcal{N}(m,C)$, we have $H(\pi)=C^{-1}$ so that the right-hand side of \eqref{eq:dimCDR_tilted} vanishes for any $U_r$ such that $U_r^\top C(\pi) U_r = I_r$. This is unsurprising since, in this case, $\pi=\mu = \pioptKLr(U_r)$ for any $U_r$.

\subsection{Numerical example: Rosenbrock with best Gaussian reference measure~$\mu$}

We revisit the two-dimensional Rosenbrock distribution from Section~\ref{sec:rosenbrock1} with the reference measure~$\mu = \mcN(m(\pi), C(\pi))$. Figure~\ref{fig:bananapiopt-bestgaussianprior} depicts the density of the optimal Gaussian reference, with highly anisotropic covariance structure, as well as an approximation~$\pioptKLr(U_1)$ with	arbitrarily chosen feature~$U_1$. Observe that the approximation retains the tail properties from the reference measure.

In Figure~\ref{fig:bananaKL-loss-bestgaussianprior} we compare the approximation properties when using the standard reference measure $\mu = \mcN(0,I)$ compared to the optimal Gaussian reference~$\mu = \mcN(m(\pi), C(\pi))$. The red line shows that the optimal reference achieves uniformly lower approximation error than with the standard reference measure shown by the blue line. However, in comparing the performances of the majorants, we note that the majorant in Theorem~\ref{thm:dimCDR_tilted}, shown by the upper red envelope,
suggests selecting the globally sub-optimal feature direction at $\theta = 180$ degrees. In contrast, the minimizer of the majorant from Theorem~\ref{thm:dimCDR}, shown by the upper blue envelope, certifies a lower approximation error, while also coinciding with the globally optimal direction, as previously discussed in Section~\ref{sec:rosenbrock1}. It is an interesting research direction to understand when one should or should not construct approximations based on the standard reference measure.

\begin{figure}[!h]
\centering\includegraphics[width=0.95\textwidth]{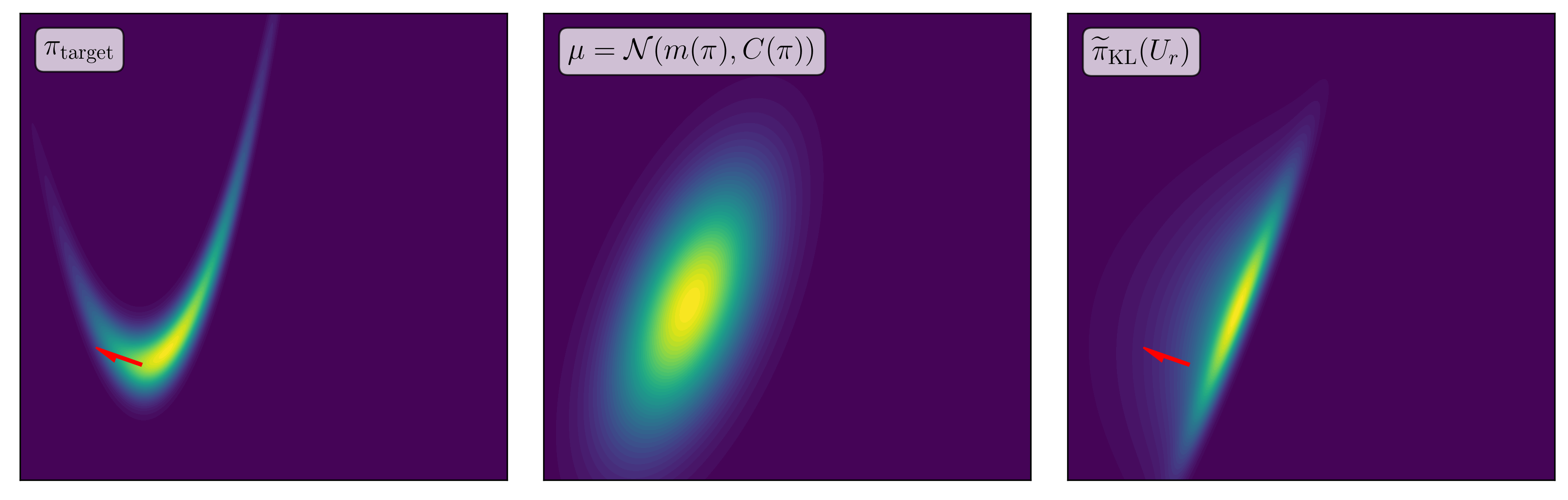}
\caption{ (Left) Probability density of target measure. (Centre) Optimal Gaussian prior with mean $m(\pi)$ and covariance $C(\pi)$. (Right) Approximate measure with feature~$U_r$ using optimal Gaussian prior.}
\label{fig:bananapiopt-bestgaussianprior}
\end{figure}

\begin{figure}[!h]
\centering\includegraphics[width=0.8\textwidth]{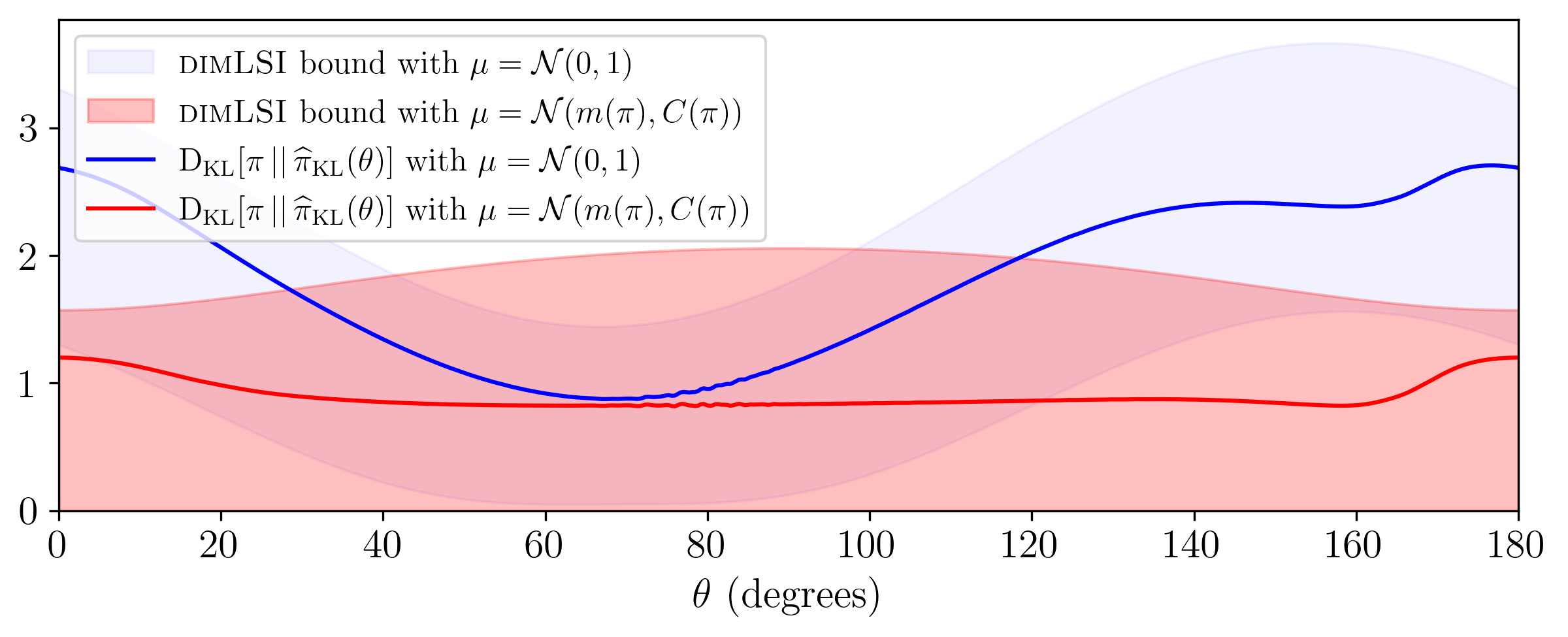}
\caption{(Rosenbrock) KL divergence $\Dkl(\pi|| \pioptKLr(U_1(\theta)))$ with $U_1(\theta)=(\cos(\theta),\sin(\theta))$ for $\theta\in[0^\circ,180^\circ)$ (red line) with the optimal Gaussian reference measure $\mu = \mathcal{N}(m(\pi), C(\pi))$ and its bounds obtained by the dimensional LSI of Theorem \ref{thm:dimCDR_tilted} shown by the shaded red region. We include for comparison the true approximation error (blue line) and dimensional LSI bounds (blue shade) in Theorem~\ref{thm:dimCDR} when constructing an approximation using the standard reference measure $\mu = \mathcal{N}(0,1)$.}
\label{fig:bananaKL-loss-bestgaussianprior}
\end{figure}

\subsection{Application to Bayesian inverse problems: the data-free setting}
\label{sec:bayesian}

For Bayesian inverse problems, the target measure is the posterior
$\d\pi^y(x) \propto \ell^y(x)\d\mu(x)$, where the likelihood function $x\mapsto\ell^y(x)$ is proportional to the probability density of observing a given data $y$ knowing $x$.
As noticed in \cite{SSC+15, ZCLSM22, Cui_Tong_2021}, in many practical settings the data~$y$ only provide information about a few linear features of~$x$. Choosing the prior~$\mu$ as the reference measure is natural here, and $U_r$ can be interpreted as the directions which capture the prior-to-posterior update.
The methodology in~\S\ref{sec:logSob} and \S\ref{sec:dimlogSob} provides an algorithm towards circumventing the impact of high-dimensionality in posterior sampling.

The \emph{data-free} approach proposed in~\cite{Cui_Zahm_2021} consists of searching for low-dimensional features~$V_r \in \R^{d \times r}$ which are informed by the \emph{average} realization of data. Formally, these features are characterized by
\begin{equation}
\label{eq:bayesloss}
\min_{\substack{V_r \in \R^{d \times r} \\ V_r^\trans V_r = I_r}} \, \E_{Y} \left[ \, \min_{\ell_r^Y:\,\R^r \to \R_+} \Dkl(\pi^Y \,||\, \widetilde\pi(V_r,\ell_r^Y)) \right] \quad \text{for }\quad \d\widetilde\pi(x \mid V_r,\ell_r^y) \propto \ell_r^y(V_r^\top x) \d\mu(x) \;.
\end{equation}
In contrast to~\eqref{eq:loss}, optimality is quantified with respect to the \emph{averaged} KL divergence with respect to data $Y$, with marginal law given by $\d\rho(y) \propto (\int \ell^y(x)\d\mu(x) ) \d y$.
The authors demonstrate an application of the logarithmic Sobolev inequality to derive a majorization to~\eqref{eq:bayesloss}. We show in Theorem~\ref{thm:dimCDRdatafree} that a tighter majorization can be attained by leveraging the dimensional LSI instead.
 The proofs of these statements are left to Appendix~\ref{sec:dimCDRdatafreeproof}.

\begin{theorem}
\label{thm:dimCDRdatafree}
Let $\d\pi^Y(x) \propto \ell^Y(x)\d\mu(x)$ be the posterior distribution of the random vector $X\sim\mu$ conditioned on~$Y$, where $\mu=\mathcal{N}(0,I_d)$ is the standard Gaussian prior and $\ell^y(x)$ is the likelihood function of observing $Y=y$ given $X=x$.
For any matrix $V_r \in \R^{d \times r}$ with $r \leq d$ orthonormal columns,
the measure $\d\widetilde\pi^{\KL,y}(x \mid V_r) \propto \ell^{\KL,y}_r(V_r^\top x) \d\mu(x)$ with $\ell^{\KL,y}_r(\theta_r) = \E_{X \sim \mu}[ \ell^y(X) \mid V_r^\top X = \theta_r]$ satisfies
\begin{equation}\label{eq:dimCDRdatafree}
\E_Y\left[ \Dkl(\pi^Y || \widetilde\pi^{\KL,Y}(V_r))  \right] \leq
\frac{1}{2} \ln\det \lb V_\perp^\top (I_d + H_\mathrm{df}) V_\perp \rb,
\end{equation}
where $V_\perp \in \R^{d \times (d-r)}$ denotes any orthonormal completion to $V_r$ and
where $H_\mathrm{df}\in\R^{d\times d}$ is the \emph{data-free diagnostic matrix} defined by
\begin{equation}
H_\mathrm{df} = \E_{X,Y}[\nabla_x \ln \ell^Y(X) \nabla_x \ln \ell^Y(X)^\top].
\end{equation}
In particular, the minimizer $V_r^\downarrow$ of the right-hand side of \eqref{eq:dimCDRdatafree} is given by $V_r^\downarrow=[v_1,\hdots,v_r]$, where $v_k$ is the eigenvector of $H_\mathrm{df}$ corresponding to the $k$-th largest eigenvalue. Denoting by $\lambda_k(H_\mathrm{df})$ the corresponding eigenvalue, \eqref{eq:dimCDRdatafree} yields
\begin{equation}
\label{eq:datafreeError}
 \E_Y\left[ \Dkl(\pi^Y \,||\, \pioptKLr\,^Y(V_r))  \right] \leq \frac{1}{2} \sum_{k=r+1}^d \ln(1 + \lambda_k(H_\mathrm{df})).
\end{equation}
\end{theorem}

The nomenclature `data-free' refers to the fact that~$H_\df$ only requires an expectation over the joint law of $(X,Y)$ and does not depend on the observed realization $y$ of the data. Samples from the joint $(X,Y)$ can be easily obtained via Gibbs algorithm, by first sampling $X\sim\mu$ and then $Y \mid X\sim\pi^X$ where $\d\pi^x(y)\propto \ell^y(x)\d y$.

Compared to the majorant $ \E_Y[ \Dkl(\pi^Y || \widetilde\pi^{\KL,Y}(V_r)) ] \leq \frac{1}{2}\trace ((I_d - V_rV_r^\top) H_\df) $ obtained by \cite[\S 3.3]{Cui_Zahm_2021},~\eqref{eq:dimCDRdatafree} is strictly tighter by the log-determinant bound $\ln\det(I+A)\leq\trace(A)$. While both majorants are minimized by the same feature $V_r^\downarrow$, the certificate of Cui and Zahm only ensures that
\[
\E_Y \Dkl(\pi \,||\, \widetilde{\pi}^{\KL,Y}(V_r^\downarrow)) \leq \frac{1}{2} \sum_{k = r+1}^d \lambda_k(H_\df) ,
\]
whereas Theorem~\ref{thm:dimCDRdatafree} guarantees that the approximation error can be significantly less. In some instances this can be \emph{exponentially} better, such as for small~$r$ when the sum of the trailing eigenvalues is large. Consequently, a significantly lower dimension~$r$ suffices to achieve the same error tolerance than previously considered.

\begin{remark}
Proposition~\ref{thm:dimCDRexactgauss} extends to the data-free setting, that is,
$\E_Y\left[ \Dkl(\pi^Y || \widetilde\pi^{\KL,Y}(V_r))  \right] = \frac{1}{2} \ln\det \lb V_\perp^\top (I_d + H_\mathrm{df}) V_\perp \rb,$ for any $V_r$ whenever $\pi^Y$ is Gaussian.
\end{remark}

\begin{remark}
 As noticed in \cite{Baptista_Marzouk_Zahm_2022}, $\mathbb{E}_Y \Dkl(\pi^Y || \widetilde\pi^{\KL,Y}(V_r)) = I(X;Y)-I(V_r^\top X;Y)$ where $I(X,Y)$ is the \emph{mutual information} between $X$ and $Y$. Thus, the bound \eqref{eq:dimCDRdatafree} yields a lower bound on the mutual information between $V_r^\top X$ and $Y$ given by
 $$
  I(V_r^\top X;Y)\geq I(X;Y)-\frac{1}{2} \ln\det \lb V_\perp^\top (I_d + H_\mathrm{df}) V_\perp \rb.
 $$
\end{remark}

\begin{remark}
The reverse dimensional LSI~\eqref{eq:revdimLSI} results in a lower bound (not shown) involving the \emph{data-free lower diagnostic matrix}
\[
G_\df = \E_Y[ \E_{X \mid Y}[\nabla_x \ln \ell^Y(X)] \E_{X \mid Y}[\nabla_x \ln \ell^Y(X)]^\top ].
\]
This requires computing nested expectations with respect to the laws of~$X|Y$ and $Y$.
A looser bound can be obtained by lower bounding this matrix, in the Loewner sense, by $G'_\df = \E_{X,Y}[\nabla_x \ln \ell^Y(X)]\E_{X,Y}[\nabla_x \ln \ell^Y(X)]^\top$. However, $G'_\df = 0$ under the commonly used data likelihood model $Y\mid X \sim \mathcal{N}(g(X),~\sigma_\mathrm{obs}^2)$, with parameter-to-observation model~$g$, resulting in a vacuous bound.
\end{remark}

\subsubsection{Applications to Inverse Problems with Generative Modeling Priors}

Recent approaches to Bayesian inverse problems adopt an empirical Bayes perspective, wherein the prior is assumed to be supported on a low-dimensional manifold and learned from training samples via generative modeling algorithms, such as~GANs~\cite{goodfellowGenerativeAdversarialNetworks2014, venkatakrishnanPlugandPlayPriorsModel2013, shahSolvingLinearInverse2018}.
These priors can be expressed as the pushforward
of a low-dimensional latent variable~$Z$ taking values in $\R^\kappa,~\kappa \ll d$, so that $\mu_X = \phi_\# \mu_Z$ for some $\phi : \R^\kappa \to \R^d$. {Under the assumption that the class of learned transport maps~$\phi$ is sufficiently expressive, $\mu_Z$ can be assumed to be the standard Gaussian measure without loss of generality.} This form of prior significantly reduces the complexity of sampling the posterior, as it then suffices to sample the $\kappa$-dimensional latent posterior
\begin{equation}
\d\pi_Z^y(z) \propto \ell^y(\phi(z)) \d\mu_Z(z),
\end{equation}
from which posterior samples in the model space~$X$ can be obtained using the push-forward relation~$\pi^y_X = \phi_\# \pi^y_Z$.

\begin{figure}[!h]
\centering\includegraphics[width=0.95\textwidth]{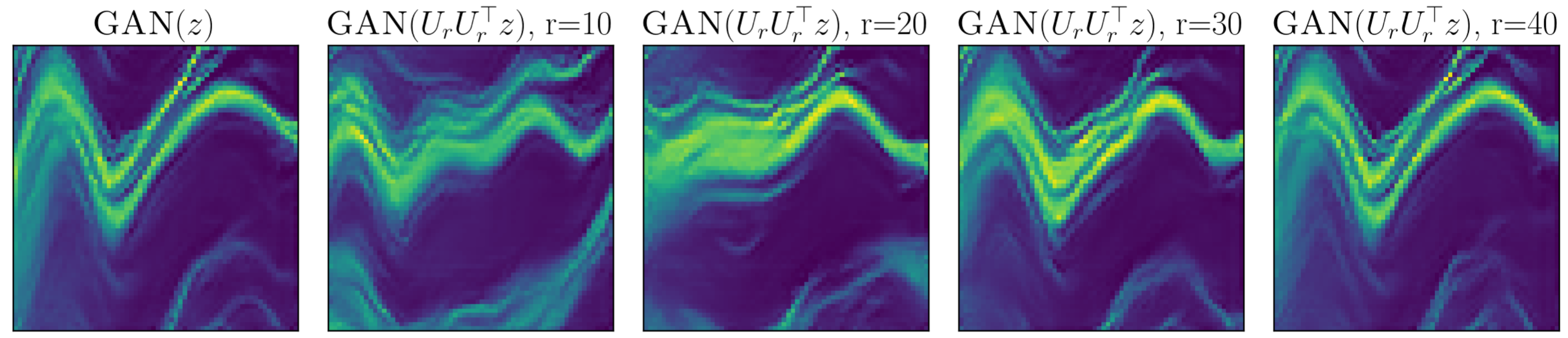}
\caption{(Left-most column) GAN decoder applied to prior sample $Z \sim \mu_Z$. (2nd to 5th columns) GAN decoder applied to same prior sample but projected onto features $U_r = V_r^\downarrow$. Increasing $r$ from left to right illustrates features in the nominal model space which are informed by data.}
\label{fig:xom_GAN_decode}
\end{figure}

We demonstrate that additional dimension reduction can be achieved by accounting for interactions between data, likelihood models, and generative priors.
{To illustrate this, we consider a geophysical inverse problem in which~$X$ represents a meshed discretization of the subsurface with $3600$ nodes.
We are provided by domain experts a trained GAN prior with a~$300$ dimensional latent space which, with some abuse of notation, we denote as the function $\textrm{GAN} : \R^{300} \to \R^{3600}$, so that the prior distribution is given by the push--forward $\textrm{GAN}_\# \mu_Z$. This prior generates realizations of the subsurface with sufficiently high fidelity for practical purposes, as exhibited by the regular, repeated layered structures seen in Figure~\ref{fig:xom_GAN_decode}.}
Using this GAN prior, we compute the error certificate~\eqref{eq:dimCDRdatafree} and visualize it in Figure~\ref{fig:xom_GAN_bound} as a function of the rank~$r$. This figure shows that data only inform relatively few latent coordinates, e.g., it suffices consider $r=50$ features to obtain an expected KL error of~$0.1$ nats.
As an additional benefit, domain practitioners can interpret the features which are ``data-informed'' by plotting the convergence of the GAN output with increasing $r$, as visualized going from the second to fifth columns in Figure~\ref{fig:xom_GAN_decode}.

\begin{figure}[!h]
\centering\includegraphics[width=0.55\textwidth]{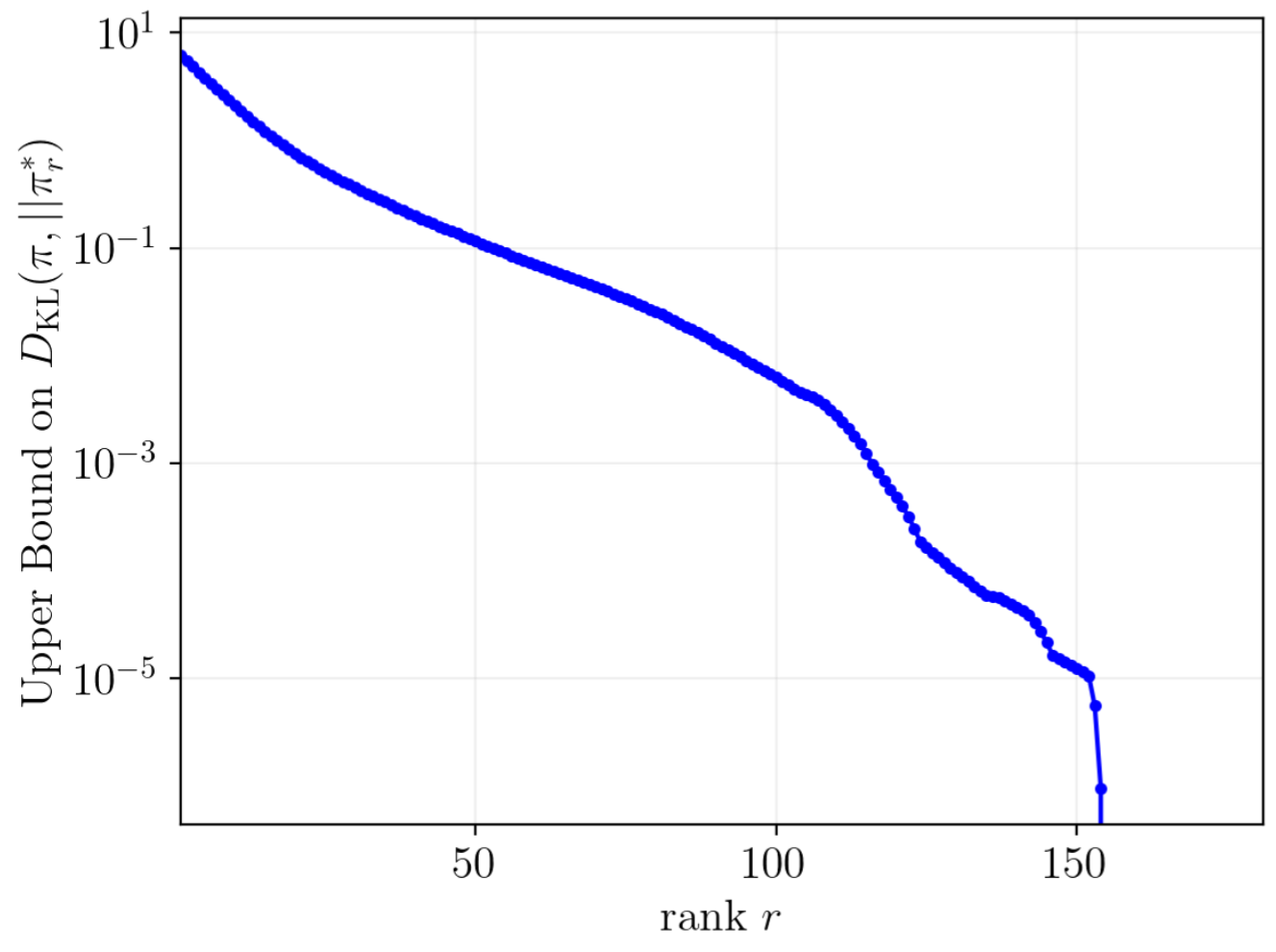}
\caption{The data-free approximation error certificate~\eqref{eq:datafreeError} versus rank~$r$ for the geophysical inverse problem with generative GAN prior.}
\label{fig:xom_GAN_bound}
\end{figure}

\section{Dimension reduction with the squared Hellinger distance}
\label{sec:Hell}

It is of interest to consider alternative error metrics in~\eqref{eq:loss} to the KL divergence.
Notably, \cite{Cui_Tong_2021, Cui_Dolgov_Zahm_2022}~considered the squared Hellinger distance
\begin{equation}
    \label{eq:defhell}
    \dhell(\pi, \mu)
    = \frac{1}{2}\int \left(\sqrt{\frac{\d\pi}{\d x}} - \sqrt{\frac{\d\mu}{\d x}} \right)^2 \d x = 1 - \int \sqrt{\frac{\d\pi}{\d\x}}\sqrt{\frac{\d\mu}{\d\x}}\d\x.
\end{equation}
{By expanding the square
and using the assumption that $\tfrac{\d\pi}{\d\mu}$ exists, we can relate this distance to the variance $\Var_\mu (\sqrt{\tfrac{\d\pi}{\d\mu} }) = 1 - \left(\int \sqrt{\tfrac{\d\pi}{\d\mu}} \d\mu \right)^{2} $ according to}
\begin{equation}\label{eq:Hell}
\dhell(\pi, \mu) = 1- \sqrt{1- \Var_\mu\lb \sqrt{\tfrac{\d\pi}{\d\mu}} \rb } .
\end{equation}
As the Poincaré inequality provides gradient-based bounds on variances, it is natural to employ it in this context to control the squared Hellinger distance (just as the logarithmic Sobolev inequality naturally relates to the KL divergence).
The Poincaré inequality, however, exhibits a ``dimension-free'' property similar to the LSI~\cite[\S3]{Chafai_2004}, and thus one may wonder whether dimensional refinements can be leveraged to obtain improved majorizations.
In this section, we discuss such refinements to the Poincar\'e inequality for the purpose of dimension reduction in the squared Hellinger distance.
Our presentation largely mirrors Section~\ref{sec:KL}.

\subsection{Certifiable bounds using the Poincar\'e Inequality}

Let us recall the Poincaré inequality for the standard Gaussian measure $\mu=\mathcal{N}(0,I_d)$.

\begin{proposition}[Theorem 4.7.2. in \cite{Bakry2014}]
\label{thm:poincare}
The standard Gaussian measure~$\mu=\mathcal{N}(0,I_d)$ satisfies
\begin{equation}
\label{eq:poincare}
\left \|\int \nabla f \d\mu \right\|_2^2 \leq \Var_\mu(f) \leq \int \| \nabla f \|_2^2\d\mu.
\end{equation}
for any smooth function $f:\R^d\rightarrow\R$.
We refer to the upper (resp. lower) bound as the (resp. reverse) \emph{Poincar\'e inequality}.
\end{proposition}

{
The squared Hellinger distance $d^2_\Hell(\pi, \mu)$ in~\eqref{eq:Hell} involves the variance of~$\sqrt{\d\pi/\d\mu}$, which can be bounded by the Poincar\'e inequality. By letting $f=\sqrt{\d\pi/\d\mu}$ in \eqref{eq:poincare} and noting $\int \|\nabla f \|_2^2\d\mu = \int \|\nabla \ln f \|_2^2 f^2\d\mu$ by the logarithmic derivative, we obtain the upper bound}
$$
 1 - \sqrt{\lb 1 - \left\| \E_\pi\left[X \sqrt{\tfrac{\d\mu}{\d\pi}} \right] \right\|^2 \rb_+} \leq
 \dhell(\pi,\mu) \leq 1 - \sqrt{\lb 1 - \frac{\trace(H(\pi||\mu))}{4}   \rb_+},
$$
where $H(\pi||\mu)=\E_{\pi}[\nabla \ln (\frac{\d\pi}{\d\mu})\nabla \ln (\frac{\d\pi}{\d\mu})^\top]$ is the Fisher information matrix of $\pi$ relative to $\mu$ and $(t)_+ := \max(t,0)$.
The lower-bound is obtained by applying Stein's identity~\cite[Lemma 1]{steinEstimationMeanMultivariate1981} to the left-hand side of~\eqref{eq:poincare} to rewrite $\int \nabla f \d\mu = \int x f(x) \d\mu(x)$, substituting the choice of $f = \sqrt{\d\pi/\d\mu}$, and applying the change of measure $\sqrt{\frac{\d\pi}{\d\mu}} \d\mu = \sqrt{\frac{\d\mu}{\d\pi}}\d\pi$.

With another judicious choice of test function~$f$, the same Poincar\'e inequality yields a tractable upper-bound for the squared Hellinger distance between $\pi$ and a structured approximation $\widetilde\pi$ as in \eqref{eq:piTildeIntro}.

\begin{theorem}[Adapted from Theorem 2.1 in \cite{Li_Marzouk_Zahm_2023} with $\alpha = 1/2$]
\label{thm:optHell}
Let $\pi$ be a probability measure on $\R^d$ and let $\mu=\mathcal{N}(0,I_d)$  be the standard Gaussian measure.
Then, for every matrix $U_r \in \R^{d \times r}$ with $r < d$ orthonormal columns, the minimizer of $\ell_r \mapsto \dhell(\pi \,||\, \widetilde\pi(U_r,\ell_r))$ where $\d\widetilde\pi(x \mid U_r,\ell_r) \propto \ell_r(U_r^\top x) \d\mu(x)$ is
\begin{equation}
\label{eq:optHell}
\ell^\Hell_r(\theta_r \mid U_r) = \E_{X \sim \mu}\left[\left. \sqrt{\frac{\d\pi}{\d\mu}}(X) \,\right|\, U_r^\top X = \theta_r \right]^2 .
\end{equation}
{
The squared Hellinger metric between~$\pi$ and the measure $\widetilde\pi^\Hell(U_r) = \widetilde\pi(U_r,\ell^\Hell_r(U_r))$ is given by
\begin{align}
\dhell(\pi, \widetilde\pi^\Hell(U_r)) &\overset{\eqref{eq:defhell}}{=} 1 - \int \sqrt{\frac{\d\pi}{\d\mu}}\,\sqrt{\frac{\d\widetilde\pi^\Hell(U_r)}{\d\mu}} \d\mu \\
&\overset{\eqref{eq:optHell}}{=} 1 - \sqrt{\int \ell_r^\Hell(x_r|U_r) \d\mu_r(x_r)} \;
\label{eq:dhelloptval} ,
\end{align}
where $\mu_r$ is the $r$-dimensional standard Gaussian measure.
}
Furthermore, this squared Hellinger metric satisfies the inequality
\begin{equation}
\label{eq:optHellLoss}
\dhell(\pi, \widetilde\pi^\Hell(U_r))
\leq
1 - \sqrt{\lb 1 - \frac{\trace(H(\pi||\mu)) - \trace(U_r^\top H(\pi||\mu)U_r)}{4}   \rb_+},
\end{equation}
where $H(\pi||\mu)=\E_{\pi}[\nabla \ln (\frac{\d\pi}{\d\mu})\nabla \ln (\frac{\d\pi}{\d\mu})^\top]$ is the Fisher information matrix of $\pi$ relative to $\mu$.
\end{theorem}

The matrix~$U_r^*=[u_1,\hdots,u_r]$ whose columns are the $r$ leading eigenvectors of $H(\pi||\mu)$ is a minimizer of the right-hand side of \eqref{eq:optHellLoss}, just as with the LSI bound for the KL divergence (c.f.~\eqref{eq:CDRerror}). In fact, \cite{Li_Marzouk_Zahm_2023} show these features certify dimension reduction for \emph{all} Amari $\alpha$-divergences with $\alpha \in (0,1]$. This furnishes~$\widetilde\pi^\Hell(U_r^*)$ with the error certificate in the squared Hellinger distance
\begin{equation}
\label{eq:CDRhell}
\dhell(\pi, \widetilde\pi^\Hell(U_r^*)) \leq 1 - \sqrt{ \lb 1 - \frac{1}{4}\sum_{k=r+1}^d \lambda_k(H(\pi||\mu))\rb_+} \;.
\end{equation}
{
In the same spirit as Prop.~\ref{thm:LSI}, one would also like to derive a lower bound to the above using the reverse Poincar\'e inequality~\eqref{eq:poincare}. However, the resulting quantity involves an explicit reference to~$Z_\pi$, and thus is difficult to accurately estimate even with access to Monte Carlo samples of~$\pi$. For this reason, we choose not to include it in our presentation.}

\subsection{Dimensional Improvements to the Poincar\'e Inequality}

There are several formulations of dimensional Poincar\'e inequalities for the standard Gaussian~$\mu$ in the literature, including:
\begin{itemize}
\item the Bobkov-Ledoux~\cite{Bobkov_Ledoux_2009} inequality
\[
\Var_\mu(f) \leq 6 \int \| \nabla f \|_2^2 \d\mu - 6 \int \frac{\| \langle \nabla f, x \rangle \|_2^2}{d + \|x\|_2^2}\d\mu,
\]
\item the Bonnefont-Joulin-Ma inequality~\cite{Bonnefont_Joulin_Ma_2016}
\[
\Var_\mu(f) \leq \frac{d(d+3)}{d-1} \int \frac{\|\nabla f\|_2^2}{1+\|x\|_2^2}\d\mu,
\]
\item and the first inequality of Bolley-Gentil-Guillin~\cite{Bolley_Gentil_Guillin_2018}, who obtain
\[
\Var_\mu(f) \leq \int \|\nabla f\|_2^2 \d\mu - \frac{1}{2d} \lb \int (\|x\|_2^2  - d)f \d\mu \rb^2
\]
by linearizing the dimensional logarithmic Sobolev inequality.\footnote{Linearizing refers to the choice of test function $f_\epsilon = 1 + \epsilon g$, where $\int g \d\mu = 0$. Substituting into~\eqref{eq:dimLSI}, scaling~$\epsilon \to 0$, and matching terms of order $\epsilon^2$ recovers the desired inequality. A similar exercise applied to the LSI~\eqref{eq:LSI} recovers the Poincar\'e inequality~\eqref{eq:poincare}.}
\end{itemize}
These inequalities are not mutually comparable as one does not imply the other.

Although we do not explicitly document this, the Bobkov-Ledoux and the Bonnefont-Joulin-Ma inequalities produce worse majorizations than~\eqref{eq:optHellLoss} obtained from the non-dimensional Poincar\'e inequality~\eqref{eq:poincare}. We speculate this is due to their large constants --- whereas~\eqref{eq:poincare} is sharp for $f(x) = \sum_{k=1}^d x_k$, it is unclear whether the prefactors of both these inequalities can be further optimized.
Furthermore, it can be shown that the Bolley-Gentil-Guillin inequality produces a majorant which explicitly requires the normalizing constant of~$\pi$, and is thus of commensurate difficulty to compute when $\pi$ is known up to a constant.

Instead, we rely on another dimensional improvement to the Poincar\'e inequality proposed by Bolley, Gentil, and Guillin.
\begin{proposition}[Theorem 4.3 in \cite{Bolley_Gentil_Guillin_2018}]
\label{thm:dimpoincare}
The variance of a smooth function $f:\R^d \to \R$ with respect to the standard Gaussian~$\mu=\mathcal{N}(0,I_d)$
satisfies the \emph{dimensional Poincar\'e inequality}
\begin{equation}
\label{eq:dimpoincare}
\Var_\mu(f) \leq \int\| \nabla f \|_2^2 \d\mu - \int \frac{\|f - \int f\d\mu - \langle \nabla f, x \rangle \|_2^2}{d+\|x\|_2^2} \d\mu.
\end{equation}
\end{proposition}

\begin{remark}
We are not aware of a \emph{reverse} dimensional Poincar\'e inequality. The proof technique used in~\cite[Thm 6.7.3]{Bakry2014} to derive the reverse dimensional LSI relies on an algebraic identity specific to the LSI entropy function $x \mapsto x \ln x$, which cannot be adapted here. Alternatively, linearizing the reverse dimensional LSI merely recovers the (non-dimensional) reverse Poincar\'e inequality in Proposition~\ref{thm:poincare}.
\end{remark}

Evidently~\eqref{eq:dimpoincare} improves on the Poincar\'e inequality, although the difference vanishes in the limit as $d \to \infty$. The dimensional Poincar\'e inequality~\eqref{eq:dimpoincare} is also tight for the test function $f(x) = \sum_{k=1}^d x_k$, which can be seen from direct calculation.

We apply the dimensional Poincar\'e inequality to obtain a majorization for the squared Hellinger approximation error. The proof can be found in Appendix~\ref{sec:dimCDRhellproof}.

\begin{theorem}
\label{thm:dimCDRhell}
Under the same setting as Theorem \ref{thm:optHell}, we have
\begin{equation}
\label{eq:dimCDRhell}
\dhell(\pi, \,\widetilde\pi^\Hell(U_r))
\leq
1 - \sqrt{\lb 1 - \frac{\trace(H(\pi||\mu)) - \trace(U_r^\top H(\pi||\mu)U_r)}{4}  + \delta_r(U_r, \,\dhell(\pi, \,\widetilde\pi^\Hell(U_r)) ) \rb_+}
\end{equation}
with
\begin{equation}
\label{eq:dimhelldelta}
\delta_r(U_r, y) = \frac{(1-(1-y)^2 - \frac{1}{2}\trace(M(\pi)) +\frac{1}{2}\trace(U_r^\top M(\pi) U_r) +\frac{1}{2}(d-r))^2}{\trace(M(\pi)) - \trace (U_r^\top M(\pi) U_r) + d-r},
\end{equation}
where $H(\pi||\mu)=\E_{\pi}[\nabla \ln (\frac{\d\pi}{\d\mu})\nabla \ln (\frac{\d\pi}{\d\mu})^\top]$ is the Fisher information matrix of $\pi$ relative to $\mu$ and $M(\pi) = \mathbb{E}_\pi[XX^\top]$ the second moment matrix.
\end{theorem}

\begin{remark}
 For an isotropic measure $\pi$ such that $M(\pi)=I_d$ we obtain
\[
\delta_r(U_r, y) = \frac{(1-(1-y)^2))^2}{2(d-r)}.
\]
\end{remark}

Observe that the function $\R^{d \times r} \times \R \ni (U_r, y) \mapsto \delta_r(U_r, y) \in \R$ defined in~\eqref{eq:dimhelldelta} is strictly positive --- the numerator is evidently positive, while the denominator is positive since $M(\pi) \succ 0$. By inspection, this implies that the majorant~\eqref{eq:dimCDRhell} improves upon the majorant~\eqref{eq:optHellLoss}. Unfortunately,  the squared Hellinger loss appears in both the left and right hand side of~\eqref{eq:dimCDRhell}, and is therefore difficult to optimize to determine linear features~$U_r$. Instead, one might hope to \emph{evaluate}~\eqref{eq:dimCDRhell} with the features~$U_r^*$ minimizing~\eqref{eq:optHellLoss} to obtain a tighter certificate for the approximation error. Surprisingly, to do so one cannot simply bootstrap the upper bound from~\eqref{eq:optHellLoss} to further bound the right hand side of~\eqref{eq:dimCDRhell}--- instead, one requires a~\emph{lower bound} on the approximation error in order to leverage the improvements from~\eqref{eq:dimCDRhell}.
Although non-vacuous lower bounds can be obtained from applying the reverse Poincar\'e inequality~\eqref{eq:poincare}, we do not expect such bounds to be easily computable.
It therefore remains an open question as to how we can make use of the improvements brought on by the dimensional Poincar\'e inequality, and we
leave exploring this direction to future work.

\appendix

\section{Necessary and sufficient conditions for first order optimality of~\eqref{eq:dimCDR_up}}
\label{sec:firstorder}

Encoding the orthonormality constraint into~\eqref{eq:dimCDR_up} with Lagrange multiplier $S=S^\top \in \R^{r \times r}$ and ignoring constant terms yields the Lagrangian
\[
L(U_r, S) = \frac{  \trace\lb U_r^\top M(\pi) U_r \rb + \ln\det\lb U_r^\top H(\pi)^{-1} U_r \rb + \trace \lb S (I_r - U_r^\top U_r)\rb }{2}.
\]
First-order optimality asserts that
\begin{equation}
\label{eq:lagrangefirst}
\nabla_{U_r} L(U_r, S) = M(\pi)U_r + H(\pi)^{-1} U_r \lb U_r^\top H(\pi)^{-1} U_r \rb^{-1} - U_r S = 0
\end{equation}
for all orthonormal $U_r$. Left multiplying~\eqref{eq:lagrangefirst} by $U_r^\top$ thus recovers
$
S = U_r^\top M(\pi) U_r + I_d,
$
and~\eqref{eq:lagrangefirst} is equivalently
\begin{equation}
\label{eq:fixedpoint}
U_r = (I_d - U_rU_r^\top)M(\pi) U_r + H(\pi)^{-1} U_r \lb U_r^\top H(\pi)^{-1} U_r \rb^{-1}.
\end{equation}
Re-arranging~\eqref{eq:fixedpoint} also gives the alternative fixed point condition
\begin{equation}
\label{eq:fixedpoint2}
U_r = H(\pi)\lb I_d - U_\perp U_\perp^\top M(\pi)\rb U_r (U_r^\top H(\pi)^{-1} U_r).
\end{equation}
We are not aware of any closed form solutions to~\eqref{eq:fixedpoint} or~\eqref{eq:fixedpoint2}.

{Left-multiplying \eqref{eq:fixedpoint2} by the transpose of the orthonormal matrix $U = [U_r, \; U_\perp]$, we obtain the system of equations}
\begin{align}
U_r^\top \left( H(\pi)( I_d - U_\perp U_\perp^\top M(\pi)) \right) U_r (U_r^\top H(\pi)^{-1}U_r) &= I_r \label{eq:critA} \\
U_\perp^\top \left( H(\pi)( I_d - U_\perp U_\perp^\top M(\pi)) \right) U_r &= 0 \; . \label{eq:critB}
\end{align}
Equation~\eqref{eq:critA} implies that $U_r^\top \left( H(\pi)( I_d - (I_d - U_rU_r^\top) M(\pi)) \right) U_r $ is symmetric, as it is the inverse of the symmetric matrix $U_r^\top H(\pi)^{-1} U_r$. Writing out this symmetry condition and re-arranging slightly obtains
\[
U_r^\top (H(\pi)M(\pi) - M(\pi) H(\pi)) U_r = (U_r^\top H(\pi) U_r)(U_r^\top M(\pi) U_r) - (U_r^\top M(\pi) U_r)(U_r^\top H(\pi) U_r) \;.
\]
{Introducing the matrix commutator notation $[A,B] = AB-BA$, expanding $H(\pi) = H(\pi||\mu)+2I_d - M(\pi)$ by~\eqref{eq:FIMidentity}, and noting $[M(\pi)-2I_d,\, M(\pi)] = 0 = [U_r^\top(M(\pi)-2I_d)U_r, \, U_r^\top M(\pi) U_r]$, the expression above can be succinctly written as
\begin{equation}
U_r^\top [H(\pi||\mu), M(\pi)]\, U_r = [U_r^\top H(\pi||\mu) U_r,\, U_r^\top M(\pi) U_r] \;. \label{eq:critA2}
\end{equation}
Applying the same expansion for $H(\pi)$ into~\eqref{eq:critB}, we have
\begin{equation}
\label{eq:critB2}
    U_\perp^\top (H(\pi||\mu)-M(\pi)) U_r - U_\perp^\top (H(\pi||\mu)+2I_d - M(\pi)) U_\perp U_\perp^\top M(\pi) U_r = 0
\end{equation}
}

If $U_r$ is a critical point of~\eqref{eq:dimCDR_up}, then it necessarily must satisfy~\eqref{eq:critA2}. Equation~\eqref{eq:critA2} can be satisfied if $U_r$ contains any $r$ eigenvectors of $H(\pi||\mu)$ or $M(\pi)$, and we examine both cases below:
\begin{enumerate}
\item suppose $U_r$ contains any $r$ eigenvectors of the second moment matrix~$M(\pi)$, so that $M(\pi)U_r = U_r S$ where $S = \mathrm{diag}(s_1, \ldots, s_r)$. Then $U_\perp^\top M(\pi) U_r = 0$ and~\eqref{eq:critB2} becomes
\[
U_\perp^\top H(\pi||\mu) U_r = 0\,.
\]
In general, this condition cannot be satisfied as vectors cannot be simultaneously orthogonal, $M(\pi)$-orthogonal, and $H(\pi||\mu)$-orthogonal, unless any pairs of these matrices commute.
\item suppose $U_r$ contains any $r$ eigenvectors of $H(\pi||\mu)$ such that $H(\pi||\mu)U_\perp = U_\perp T_\perp$, {where $T_\perp$ is the diagonal matrix of the corresponding eigenvalues}. Since $H(\pi||\mu)$ and $M(\pi)$ can be computed numerically for Gaussian~$\pi$ and~$\mu$, we computed these eigenvectors numerically and substituted into~\eqref{eq:critB2}. In general, we observed that this necessary condition is not satisfied.
\end{enumerate}
Both observations suggest that although the leading eigenvectors of the relative Fisher information matrix are globally optimal for~\eqref{eq:CDRbound}, they are not even local minima of~\eqref{eq:dimCDR_up}, even when~$\pi$ and~$\mu$ are Gaussian. As such, since closed form solutions are not available,  numerical optimization is required to obtain critical points of~\eqref{eq:dimCDR_up}.

\begin{remark}
Observe that~\eqref{eq:critA} allows the component $(U_r^\top H(\pi)^{-1} U_r)$ appearing in~\eqref{eq:dimCDR_up} to be replaced by a function involving only $H(\pi)$ and $M$ at optimality, which does not require computing matrix inverses. This may offer computational speedups, but we do not investigate this further in the current work.
\end{remark}

\section{Auxiliary Matrix Determinant Identity}
\label{sec:matrixdeterminant}

For any matrix $U_r \in \R^{d \times r}$ with orthonormal columns, and any orthonormal complement $U_\perp$ to $U_r$, Sylvester's determinant theorem allows us to write
\begin{align*}
&\det \lb U_\perp^\top H(\pi) U_\perp \rb \\
& = \det \lb U_\perp^\top (H(\pi) - I_d) U_\perp + I_{d-r}\rb  = \det \lb (H(\pi) - I_d) (U_\perp U_\perp^\top) + I_d \rb \\
& = \det \lb (H(\pi) - I_d) (I_d - U_r U_r^\top) + I_d \rb  = \det \lb H(\pi) -  H(\pi) \, U_r U_r^\top - (I_d - U_r U_r^\top) + I_d \rb\\
& = \det \lb H(\pi) -  H(\pi) \, U_r U_r^\top + U_r U_r^\top \rb = \det \lb H(\pi) \rb \det \lb I_d -  U_r U_r^\top +  H(\pi)^{-1} U_r U_r^\top \rb\\
& = \det \lb H(\pi) \rb \det \lb I_d +  (H(\pi)^{-1} U_r - U_r)U_r^\top \rb = \det \lb H(\pi) \rb \det \lb I_r +  U_r^\top (H(\pi)^{-1} U_r - U_r) \rb \\
& = \det \lb H(\pi) \rb \det \lb U_r^\top H(\pi)^{-1} U_r \rb,
\end{align*}
assuming $H(\pi)$ is invertible.

\section{Derivation of Functional Inequalities}

\subsection{Gaussian Logarithmic Sobolev Inequality}
\label{sec:proofLSI}
Numerous derivations of the Gaussian logarithmic Sobolev inequality are known, but we prefer the presentation of \cite{Bakry2014} as this also conveniently encapsulates the \emph{reverse} LSI. The following is a condensed version of \cite[Thm 5.5.2]{Bakry2014}, which is sufficient to obtain Proposition~\ref{thm:LSI}.
\begin{theorem}[Local Logarithmic Sobolev Inequality]
\label{thm:localLSI}
A Markov diffusion process $\{X_t,~t \geq 0\}$ satisfies a curvature dimension inequality $\textsc{CD}(0,\infty)$ if and only if it satisfies
the \emph{local logarithmic Sobolev inequality}
\begin{equation}
\label{eq:localLSI}
P_t(f \ln f) - P_tf \ln P_tf
\leq t P_t \lb \frac{\Gamma f}{f}\rb
\end{equation}
and the \emph{reverse local logarithmic Sobolev inequality}
\begin{equation}
\label{eq:revlocalLSI}
P_t(f \ln f) - P_tf \ln P_tf \geq t\, \frac{\Gamma(P_t f)}{P_tf }
\end{equation}
for smooth positive functions $f:\R^d \to \R_+$.
\end{theorem}

We emphasize that the presentation as
\emph{local} logarithmic Sobolev inequalities,
where `local' refers to the law of the diffusion processes at time $t < \infty$,
is key to obtaining the reverse inequality. Although it is
customary to present the ergodic ($t \to \infty$) formulation of the logarithmic Sobolev inequality, this limit leads to a vacuous reverse inequality.

The central idea towards converting Eqs~\eqref{eq:localLSI} and \eqref{eq:revlocalLSI} into Proposition~\ref{thm:LSI} is to specialize to $d$-dimensional Brownian motion
\begin{equation}
\label{eq:heatprocess}
\d X_t = \d W_t, \quad X_0 = x_0 \;.
\end{equation}
While Brownian motion does not have an invariant measure, at any time $t > 0 $ the law of $X_t$ is isotropic Gaussian with mean~$x_0$ and marginal variance~$t$. It is straightforward to verify that this Markov diffusion process satisfies the $\mathrm{CD}(0,d)$ curvature dimension inequality\footnote{also synonymously referred to as the \emph{Bakry-\'Emery criterion}.}; see \cite[Definition 1.16.1]{Bakry2014} for precise details. In addition, note that
if a diffusion process satisfies $\textsc{CD}(\rho, n)$, then it also satisfies $\textsc{CD}(\rho', n')$ for any $n' \geq n$ and $\rho' \leq \rho$, hence standard Brownian motion satisfies the local logarithmic Sobolev inequality.

The \emph{Markov semigroup} operator of Brownian motion is given by
$
P_tf(x) := \E[ f(X_t) \mid X_0 = x] = \E_{Z \sim \mcN(0,I_d)}[f(x+tZ)],
$
while its \emph{carr\'e du champ} operator is
\begin{equation}
\label{eq:carreduchamp}
\Gamma f := \Gamma(f,f) = \frac{1}{2} \| \nabla f \|_2^2.
\end{equation}
Accordingly, fix $t = 1$ and, without loss of generality, consider the initial condition $x_0 = 0$ since, alternatively, we can choose the test functions $g(\cdot) = f(\cdot - x_0)$). Then, the law of $X_1 | X_0 = x_0$ is the $d$ dimensional isotropic Gaussian~$\mu$, and we have $\Ent_\mu(f) = \Ent_{P_1}(f)$, which matches the left hand side of both~\eqref{eq:localLSI} and~\eqref{eq:revlocalLSI}. Substituting the carr\'e du champ~\eqref{eq:carreduchamp} into~\eqref{eq:localLSI} and simplifying via the logarithmic derivative obtains the desired upper bound~\eqref{eq:LSI}. To obtain the reverse inequality, note $P_1f(x_0) = \E_{Z \sim \mu}[f(x_0+Z)]$ so that $\nabla_x P_1f(x_0) = \E_{Z \sim \mu}[\nabla_z f(x_0+Z)]$. Evaluating the carr\'e du champ $\Gamma P_1f(0) = \frac{1}{2} \| \E_{\mu}[\nabla f] \|_2^2$ obtains~\eqref{eq:revlocalLSI} and concludes the proof.

\subsection{Dimensional Gaussian Logarithmic Sobolev Inequality}\label{sec:fun_inequalities}

A proof of the dimensional logarithmic Sobolev inequality, and its reverse, for general Markov diffusion processes satisfying specific curvature-dimension inequalities is provided in \cite[Thm 6.7.4]{Bakry2014}. Specifying to Brownian motion, as above, then recovers the weaker dimensional Gaussian LSI of~\cite{Bakry_Ledoux_2006}.

Instead, we choose to follow the original presentations in Bakry and Ledoux \cite{Bakry_Ledoux_2006} and Dembo~\cite{Dembo_1990}, as their approach clarifies the dependence on the Gaussianity of the reference measure. Neither reference documents the reverse inequality, but these authors were undoubtedly aware of its existence. Since we were unable to find this result explicitly in the literature, we provide its statement and proof below.

The proof proceeds by transforming the Gaussian LSI into Euclidean form by choosing the test function $f(x) = ((2\pi)^{-d/2}\exp(-\frac{1}{2}\|x\|_2^2))^{-1} g(x)$, with $g : \R^d \to \R_+$ normalized as $\int g(x) \d\x = 1$, so that $\int f \d\mu = 1$.
Applying this~$f$ to the inequalities in~\eqref{eq:LSI}, {and recalling $\d\mu =(2\pi)^{-d/2}\exp(-\frac{1}{2}\|x\|_2^2) \d x$}, yields the \emph{Euclidean logarithmic Sobolev inequality}
\begin{equation}
\label{eq:euclideanLSI}
\int g(x)\ln g(x) \d x \leq \frac{1}{2} \int \frac{\| \nabla g(x) \|_2^2}{g(x)}\d x - \frac{d}{2}\ln(2\pi) - d,
\end{equation}
as well as the \emph{reverse Euclidean logarithmic Sobolev inequality}
\begin{equation}
\label{eq:reveuclideanLSI}
\int g(x)\ln g(x) \d x \geq \frac{1}{2} \left\| \int x g(x)\d x \right\|_2^2 - \frac{d}{2}\ln(2\pi) - \frac{1}{2} \int g \|x\|_2^2 \d x,
\end{equation}
{where we have used that $\int \tfrac{\partial}{\partial_{x_k}} g(x) \d\x = 0$ for all $1 \leq k \leq d$ since $g$ is assumed to be $L_1$ integrable.}

The goal is to bootstrap this to obtain dimensional improvements to the Euclidean LSI. The insight of~\cite{Bakry_Ledoux_2006, Dembo_1990} is to recognize that the test functions $g$ are probability density functions. This suggests considering $g(x) = (T^{-1})_\# f(x) = |\det \nabla T(x)| f(T(x))$ for $f:\R^d \to \R_+$ and $\int f(x)\d\x=1$, i.e., the push-forward of a probability density $f$ under diffeomorphisms $T^{-1} : \R^d \to \R^d$, and optimizing the bounds with respect to $T$ (or $T^{-1}$). Specifically, Bakry and Ledoux considered the map $T(x) = \sigma x$, optimizing over the scalar parameter $\sigma > 0$. Dembo~\cite{Dembo_1990} considered the more general construction of maps $T(x) = Ax$ with symmetric positive definite matrices $A = A^\top \succ 0$.

{Applying the test function $g = (T^{-1})_\# f$  with~$T(x)=Ax$ for some symmetric~$A\succ 0$ into~\eqref{eq:euclideanLSI}}, we obtain
\[
\int f(x)\ln f(x)\d\x \leq \frac{1}{2} \trace \lb A^2 \lb \int \frac{\nabla f(x)^{\otimes 2}}{f(x)} \d\x \rb \rb - \frac{d}{2}\ln(2\pi) - d - \frac{1}{2}\ln \det(A^2) \;.
\]
{Noting that the matrix~$A$ does not appear in the left hand side, we can minimize the right hand side with respect to this degree of freedom,}
and straightforward computation shows that $A^{-2} = \int (\nabla \ln f(x))^{\otimes 2} f(x)\d\x$ achieves the minimum. Substituting this optimal choice of~$A$ above, we obtain the \emph{dimensional} Euclidean LSI
\begin{equation}
\label{eq:dimEuclLSI}
\int f(x)\ln f(x)\d\x \leq \frac{1}{2} \ln \det \lb \int \frac{\nabla f(x)^{\otimes 2}}{f(x)} \d\x \rb  - \frac{d}{2}\ln(2\pi e).
\end{equation}
The same choice of test function~$g = (T^{-1})_\# f$ applied to~\eqref{eq:reveuclideanLSI} yields
\[
\int f(x)\ln f(x)\d\x \geq \frac{1}{2} \left\| A^{-1} \int x f(x)\d\x \right\|_2^2 - \frac{d}{2} \ln(2\pi) - \frac{1}{2} \int \|A^{-1}x\|_2^2 f(x)\d\x - \ln\det A.
\]
Following the same observation as above, the lower bound is maximized with $A^2 = \int x^{\otimes 2} f(x)\d\x - \lb \int xf(x)\d\x\rb^{\otimes 2}$, and at optimality  we obtain the \emph{dimensional} reverse Euclidean LSI
\begin{equation}
\label{eq:revdimEuclLSI}
\int f(x)\ln f(x)\d\x \geq -\frac{1}{2} \ln \det \lb \int x^{\otimes 2} f(x)\d\x - \lb \int xf(x)\d\x\rb^{\otimes 2}\rb -\frac{d}{2}\ln(2\pi e).
\end{equation}

It remains to convert these inequalities for the Gaussian measure by considering test functions $f(x) = g(x) (2\pi)^{d/2}\exp(-\frac{1}{2}\|x\|_2^2)$, normalized according to $\int f(x)\d\x = 1$ so that $\int g\d\mu = 1$. Substituting this into~\eqref{eq:dimEuclLSI} obtains~\eqref{eq:dimLSI}, and~\eqref{eq:revdimEuclLSI} obtains~\eqref{eq:revdimLSI}, by applying the Gaussian integration by parts identity $\int x_k g(x)\d\mu(x) = \int \tfrac{\partial}{\partial_{x_k}} g(x)\d\mu(x)$.

\section{Miscellaneous Proofs}

\subsection{Proof of Lower Bound in Theorem~\ref{thm:CDR}}
\label{sec:lowerCDRproof}

We only detail the proof of the lower bound in Theorem~\ref{thm:CDR} below since the upper bound is extensively documented in \cite{ZCLSM22}.

\begin{proof}
{Given a matrix $U_r\in\R^{d\times r}$ with orthonormal columns and $x \in \R^d$, we can write $x = U_r x_r + U_\perp x_\perp$ for some $x_r \in \R^r$ and $x_\perp \in \R^{d-r}$, and $U_\perp$ is any orthonormal complement of~$U_r$.
Consider the reverse LSI~\eqref{eq:LSI} for the $(d-r)$ dimensional Gaussian measure $\mu_{\perp}$, which states that
\begin{equation}
\label{eq:lsilowertmp}
 \frac{1}{2 \int f \d\mu_{\perp}}\left\| \int \nabla_{x_\perp} f \d\mu_\perp \right\|^2 \leq \int f \ln f \d\mu_\perp
\end{equation}
for sufficiently smooth positive functions $x_\perp \mapsto f(x_\perp)$.
Choose the test function
\begin{equation}\label{eq:def_f_proof}
    x_\perp \mapsto f(x_r, x_\perp)
 = \frac{\ell( U_r x_r + U_\perp x_\perp )}{\ell^\KL_r(x_r)} ,
\end{equation}
where we assume $\ell:\R^d\rightarrow\R_+$ is normalized so that $d\pi = \ell \d\mu$, and $\ell^\KL_r$ as in \eqref{eq:optKL}.
Note by definition of $\ell^\KL_r$ we have $\int f \d\mu_\perp = 1$.
With this choice of~$f$, we integrate both sides of inequality~\eqref{eq:lsilowertmp} with respect to $x_r$ against the probability measure $\d \pi_r(x_r) = \ell^\KL_r(x_r) \d\mu_r(x_r)$, where $\mu_r$ is the $r$-dimensional standard Gaussian measure.
The right-hand side~\eqref{eq:lsilowertmp} becomes
\begin{align}
\label{eq:tmp42609}
\int \lb \int f \ln f  \d\mu_{\perp } \rb \d\pi_r
&=\int   \ell(x) \ln  \frac{\ell(x)}{\ell^\KL_r(U_r^\top x)}   \d\mu(x)
 = \Dkl(\pi||\pioptKLr(U_r)) .
\end{align}
Using Jensen's inequality, the left-hand side of~\eqref{eq:lsilowertmp} satisfies
\begin{align*}
 \int \lb \frac{1}{2}\left\| \int \nabla_{x_\perp} f \d\mu_\perp \right\|^2 \rb \d\pi_r
 &\geq \frac{1}{2}\left\| \int \lb \int \nabla_{x_\perp} f \d\mu_\perp  \rb \d\pi_r  \right\|^2\\
 &=  \frac{1}{2}\left\| \int \lb \frac{1}{\ell^\KL_r} \int U_\perp^\top \nabla \ell  \d\mu_\perp  \rb \ell^\KL_r \d\mu_r  \right\|^2
 =  \frac{1}{2}\left\| U_\perp^\top \int \nabla \ln \ell  \d\pi   \right\|^2 \;,
\end{align*}
where in the first equality we applied the chain rule and the change of measure formula to~$d\pi_r$.
Noting
$$
 \int \nabla \ln \ell(x)  \d\pi(x) \overset{\d\pi=\ell\d\mu}{=} \int \nabla  \ell(x)  \d\mu(x) \overset{\text{Stein's identity} }{=} \int x \ell(x)\,\d\mu(x)  = \int x \, \d\pi(x)
$$
and combining the above relations yields the lower bound in \eqref{eq:CDRbound}.
}
\end{proof}

\subsection{Proof of Theorem~\ref{thm:dimCDR}}
\label{sec:dimCDRproof}

\begin{proof}
{Given a matrix $U_r\in\R^{d\times r}$ with orthonormal columns and $x \in \R^d$, we can write $x = U_r x_r + U_\perp x_\perp$ for some $x_r \in \R^r$ and $x_\perp \in \R^{d-r}$, and $U_\perp$ is any orthonormal complement of~$U_r$.  Consider the dimensional LSI \eqref{eq:dimLSI} with the $d-r$ dimensional standard Gaussian measure $\mu_{\perp}$, where
\begin{align}
\label{eq:tmp36289}
\int f \ln f  \d\mu_{\perp }
&\leq \frac{1}{2}\int \|x_\perp\|^2 f \d\mu_{\perp } - \frac{d-r}{2} +\frac{1}{2}\ln\det \lb \int \lb \nabla_{x_\perp} \ln f-x_\perp\rb^{\otimes 2}f \d\mu_\perp \rb \;.
\end{align}
As in the proof of Proposition~\ref{thm:LSI}, we choose the function $f:\R^{d-r}\rightarrow \R_+$ defined as in \eqref{eq:def_f_proof} with $\ell$ assumed to be normalized such that $d\pi=\ell\d\mu$.

With this choice of~$f$, we integrate the inequality above with respect to $x_r$ against the probability measure $\d\pi_r = \widetilde\ell^\KL_r \d\mu_r$, where $\mu_r$ is the $r$-dimensional standard Gaussian measure. Following the same computations as in~\eqref{eq:tmp42609}, the left-hand side of \eqref{eq:tmp36289} becomes $\int \lb \int f \ln f  \d\mu_{\perp } \rb \d\pi_r  = \Dkl(\pi||\pioptKLr(U_r))$.
For the first term in the right-hand of \eqref{eq:tmp36289}, we obtain
\begin{align*}
\int \lb \frac{1}{2}\int \|x_\perp\|^2 f \d\mu_{\perp } \rb \d\pi_r
&=\frac{1}{2}\int \|U_\perp^\top x\|^2 \ell \d\mu \overset{\d\pi=\ell\d\mu}{=}\frac{1}{2}\int \|U_\perp^\top x\|^2 \d\pi \;,
\end{align*}
where the first equality follows from the choice of~$f$ and the property~\eqref{eq:tmp42609} of the measure~$\pi_r$, and the second equality is obtained from a change of measure. For the right-most term in the right-hand side of \eqref{eq:tmp36289}, we use the concavity of the log-determinant function to obtain the inequality
\begin{align*}
\int\frac{1}{2}\ln\det \lb \int \lb \nabla_{x_\perp}\ln f-x_\perp\rb^{\otimes 2}f \d\mu_\perp \rb \d\pi_r
&\leq \frac{1}{2}\ln\det \lb\int \lb \int \lb \nabla_{x_\perp}\ln f-x_\perp\rb^{\otimes 2}f \d\mu_\perp \rb\d\pi_r \rb  \\
&= \frac{1}{2}\ln\det \lb \int \lb U_\perp^\top (\nabla\ln \ell -x) \rb^{\otimes 2} \ell\d\mu \rb \;,
\end{align*}
where the equality is obtained by the chain rule and the choice of~$f$.
Substituting the three above relations in \eqref{eq:tmp36289} yields the upper bound of \eqref{eq:dimCDRbound}. The lower bound is obtained in a similar way for the same choice of test function~$f$.
}
\end{proof}

\subsection{Proof of Proposition~\ref{thm:dim_monotone}}
\label{sec:proof_dim_monotone}
We require the following lemma.
\begin{lemma}
\label{thm:logdetinequality}
Let $v \in \R^d$ be such that $v^\top v > 0$ and $W \in \R^{d \times k}$, $1 \leq k < d$, be such that $W^\top W \succeq 0$.
Then we have the inequality
\begin{equation}
\label{eq:logdetinequality}
\ln\det v^\top v + \ln\det W^\top W \leq \ln\det \begin{pmatrix} v^\top v & v^\top W \\ W^\top v & W^\top v \end{pmatrix}.
\end{equation}
\end{lemma}

\begin{proof} Re-write~\eqref{eq:logdetinequality} as
\[
\ln\det \begin{pmatrix} v^\top v & \\ & I_{k} \end{pmatrix} + \ln\det \begin{pmatrix} 1 & \\ & W^\top W \end{pmatrix} \leq \ln\det \begin{pmatrix} v^\top v & v^\top W \\ W^\top v & W^\top v \end{pmatrix},
\]
where $I_k$ denotes the $k \times k$ identity matrix. Re-arranging then gives
\[
\ln \det \begin{pmatrix} 1 & \theta^\top \\ \theta & I_k \end{pmatrix} = \ln\det \lb I_{k+1} + \begin{pmatrix}0 \\ \theta\end{pmatrix} \begin{pmatrix} 0 & \theta^\top \end{pmatrix} \rb \geq 0
\]
with $\theta = \frac{1}{\sqrt{v^\top v}} (W^\top W)^{-\frac{1}{2}} W^\top v$. But since  $\det \lb I_{k+1} + \begin{pmatrix}0 \\ \theta\end{pmatrix} \begin{pmatrix} 0 & \theta^\top \end{pmatrix} \rb = 1 + \|\theta\|_2^2$, this concludes the proof.
\end{proof}

We now demonstrate that $U_r \mapsto \mcJ^\downarrow_{\KL}(U_r)$ is monotonically decreasing for nested sequences.
\begin{proof}[Proof of Proposition~\ref{thm:dim_monotone}]
Let us denote $U_r = \begin{pmatrix} u_1 & \ldots u_r \end{pmatrix}$ and $U_{\perp,r} = \begin{pmatrix} u_{r+1} & \ldots & u_d \end{pmatrix} \in \R^{d \times (d-r)}$ for $1 \leq r < d$; {we carry the subscript~$r$ in our notation since it will be necessary to distinguish between $U_{\perp,r}$ and $U_{\perp,r+1}$ in the following}. Then, we have
\begin{equation}
\label{eq:loss1}
\mcJ^\downarrow_{\KL}(U_r) = \frac{1}{2} \trace U_{\perp,r}^\top M(\pi) U_{\perp,r} - \frac{d-r}{2} + \frac{1}{2}\ln\det \lb U_{\perp,r}^\top \; H(\pi) U_{\perp,r} \rb
\end{equation}
and
\begin{equation*}
\mcJ^\downarrow_{\KL}(U_{r+1}) = \frac{1}{2} \trace U_{\perp,r+1}^\top M(\pi) U_{\perp,r+1} - \frac{d-r-1}{2} + \frac{1}{2}\ln\det \lb U_{\perp,r+1}^\top \, H(\pi) U_{\perp,r+1} \rb.
\end{equation*}
Demonstrating  $\mcJ^\downarrow_{\KL}(U_{r+1}) \leq \mcJ^\downarrow_{\KL}(U_r)$ amounts to the inequality
\[
1 - u_{r+1}^\top M(\pi) u_{r+1} + \ln\det \lb U_{\perp,r+1}^\top \; H(\pi) U_{\perp,r+1} \rb \leq \ln\det \lb U_{\perp,r}^\top H(\pi) U_{\perp,r} \rb.
\]
The key idea is to recognize that by defining the orthonormal sub-matrix
\begin{equation*}
\label{eq:V}
V = \begin{pmatrix} u_1 & \ldots & u_r & u_{r+2} & \ldots & u_d \end{pmatrix} \in \R^{d \times (d-1)},
\end{equation*}
we have
\[
0 \leq \mcJ^\downarrow_{\KL}(V) = \frac{u_{r+1}^\top M(\pi) u_{r+1} - 1}{2} + \frac{1}{2}\ln\det \lb u_{r+1}^\top H(\pi) u_{r+1} \rb,
\]
where non-negativity is ensured since
\[
\Dkl(\pi\,\|\,\pioptKLr(U^\downarrow_{d-1})) \overset{\eqref{eq:dimCDRbound}}{\leq} \min_{U_{d-1}^\top U_{d-1} = I_{d-1}} \mcJ^\downarrow_{\KL}(U_{d-1}) \leq \mcJ^\downarrow_{\KL}(V).
\]
It therefore suffices to show
\[
\ln\det \lb u_{r+1}^\top H(\pi) u_{r+1}\rb  + \ln\det \lb U_{\perp, r+1}^\top \; H(\pi) U_{\perp, r+1}\rb   \leq \ln\det \lb U_{\perp, r}^\top \; H(\pi) U_{\perp, r}\rb \;,
\]
which follows from applying Lemma~\ref{thm:logdetinequality} with $v = H(\pi)^{\frac{1}{2}} u_{r+1}$ and $W = H(\pi)^{\frac{1}{2}}U_{\perp,r+1}$.
\end{proof}

\subsection{Proof of Theorem \ref{thm:dimCDR_tilted}}\label{proof:dimCDR_tilted}

\begin{proof}
{
 Let $T(x)=C(\pi)^{-1/2}(x-m(\pi))$ be the affine map such that $\pi_0 = T_\sharp\pi$ is centered and isotropic, that is, $m(\pi_0)=0$ and $M(\pi_0)=I_d$. For any $U_r\in\R^{d\times r}$ such that $U_r^\top C(\pi) U_r=I_r$, we let $W_r=C(\pi)^{1/2}U_r$ so that $W_r^\top W_r=I_d$.
 Applying Theorem \ref{thm:dimCDR} to $\pi_0$ yields
 \begin{equation}\label{eq:tmp1350427}
  \Dkl( \pi_0 || \pioptKLr_0( W_r ) ) \leq \frac{1}{2}\ln\det (  H(\pi_0)  ) - \frac{1}{2}\ln\det ( W_r^\top H(\pi_0) W_r )  ,
 \end{equation}
 where $\pioptKLr_0(W_r)=\pioptKLr_0(\cdot\mid W_r)$ is defined by $\d\pioptKLr_0(x \mid W_r) = \ell^{\KL}_{0,r}( W_r^\top x ) \d\mu_0(x)$, where $\mu_0=\mathcal{N}(0,I_d)$ and
 \begin{equation}\label{eq:tmp10579}
     \ell^{\KL}_{0,r}(\theta_r ) = \E_{X \sim \mu_0}\left[ \frac{\d\pi_0}{\d\mu_0}(X) \Big|  W_r^\top X = \theta_r \right] .
 \end{equation}
 Next, we show that \eqref{eq:tmp1350427} yields \eqref{eq:dimCDR_tilted}.
 First, because $T^{-1}(x) = C(\pi)^{1/2}x + m(\pi)$ is affine, we have that $|\det\nabla T^{-1}(x)|=|\det C(\pi)^{1/2}|$ is constant so that $\d\pi_0(x) \propto \d\pi(T^{-1}(x))$. Thus we can write
 $$
  H(\pi_0) =
  \E_{\pi_0}\left[ C(\pi)^{1/2}\nabla\ln\left(\frac{\d\pi}{\d x}\circ T^{-1}\right)\nabla\ln\left(\frac{\d\pi}{\d x}\circ T^{-1}\right)^\top C(\pi)^{1/2} \right]
  = C(\pi)^{1/2} H(\pi) C(\pi)^{1/2} .
 $$
 Because $V_r=C(\pi)^{1/2} W_r$, we deduce that the right-hand sides of \eqref{eq:tmp1350427} and \eqref{eq:dimCDR_tilted} are identical.
 To show that the left-hand sides are identical too, we notice that
 $$
  \Dkl( \pi || \pioptKLr(U_r))
  = \Dkl( T_\sharp \pi || T_\sharp \pioptKLr(U_r)  )
  = \Dkl( \pi_0 || \pioptKLr_0( W_r) ),
 $$
 where we used the fact that
 \begin{align*}
  T_\sharp \pioptKLr( \cdot | U_r)(x)
  &\propto  \pioptKLr( T^{-1}(x) | U_r) \\
  (\text{Equation }\eqref{eq:pioptKL})\quad&\propto  \ell^{\KL}_r( U_r^\top T^{-1}(x) ) \mu( T^{-1}(x) ) \\
  (\text{Equation }\eqref{eq:optKL})\quad&\propto  \E_{Y \sim \mu}\left[\left. \frac{\d \pi }{\d \mu  }( Y ) \right| U_r^\top Y = U_r^\top T^{-1}(x) \right]  \mu_0(x) \\
  (Y \leftarrow T^{-1}(X))\quad&\propto  \E_{X \sim \mu_0}\left[\left. \frac{\d\pi_0 }{\d\mu_0  }( X ) \right| U_r^\top T^{-1}(X) = U_r^\top T^{-1}(x) \right]  \mu_0(x)  \\
  (W_r=C(\pi)^{1/2}U_r) \quad&\propto  \E_{X \sim \mu_0}\left[\left. \frac{\d\pi_0 }{\d\mu_0  }( X ) \right| W_r^\top X = W_r^\top x \right]  \mu_0(x)  \\
  (\text{Equation }\eqref{eq:tmp10579})\quad&\propto \ell^{\KL}_{0,r}(W_r^\top x)\mu_0(x) \\
  &\propto  \pioptKLr_0( x| W_r)
 \end{align*}
 Finally, evaluating \eqref{eq:dimCDR_tilted} at $U_r = U_r^\downarrow$ yields \eqref{eq:dimCDR_tilted_optimal} and concludes the proof.}
\end{proof}

\subsection{Proof of Theorem~\ref{thm:dimCDRdatafree}}
\label{sec:dimCDRdatafreeproof}
\begin{proof}
{
For every fixed realization of random variable~$Y$, consider Theorem~\ref{thm:dimCDR}, formulated using~\eqref{eq:dimCDRperp_up}, with target measure~$\pi^Y$ and approximation~$\widetilde{\pi}^Y(V_r) := \widetilde{\pi}^{\KL,Y}(V_r)$ for any orthonormal sub-matrix~$V_r \in \R^{d \times r}$. Upon substituting~\eqref{eq:FIMidentity}, this yields the inequality
\[
\Dkl(\pi^Y \,||\, \widetilde\pi^Y(V_r)) \leq \frac{1}{2} \trace( V_\perp^\top M^Y V_\perp ) - \frac{d-r}{2}
+ \frac{1}{2} \ln\det \lb V_\perp^\top (2 I_d + H^Y-M^Y) V_\perp \rb,
\]
where $M^Y := \E_{X | Y \sim \pi^Y}[XX^\top]$ and $H^Y := \E_{X|Y \sim \pi^Y}[ \nabla \ln \ell^Y(X) \nabla \ln \ell^Y(X)^\top]$ denote matrices which are measurable functions of the random variable~$Y$. Integrating both sides of the inequality with respect to the law $\rho$ of this random variable, and applying Jensen's inequality to the concave function $S \mapsto \ln\det S$, we obtain
\[
\E_{Y \sim \rho}[\Dkl(\pi^Y \,||\, \widetilde\pi^Y(V_r))] \leq \frac{1}{2} \trace( V_\perp^\top \E_{Y \sim \rho}[M^Y] V_\perp ) - \frac{d-r}{2}
+ \frac{1}{2} \ln\det \lb V_\perp^\top (2 I_d + \E_{Y \sim \rho}[H^Y-M^Y]) V_\perp \rb \;.
\]
Noting $\E_{Y \sim \rho}[M^Y] = \E_{X \sim \mu}[XX^\top] = I_d$ and $\E_{Y \sim \rho}[H^Y] = H^\df$ concludes the proof.
}
\end{proof}

\subsection{Proof of Theorem~\ref{thm:dimCDRhell}}
\label{sec:dimCDRhellproof}

The following lemma is used in the proof of Theorem~\ref{thm:dimCDRhell}.
\begin{lemma}
\label{thm:jointconvexity}
For all $\mu$-integrable functions $f$ and $g$
we have the inequality
$
\frac{\E_{X \sim \mu} [f(X)]^2}{\E_{X \sim \mu}[g(X)]}
\leq
\E_{X \sim \mu} \left[ \frac{f(X)^2}{g(X)} \right].
$

\end{lemma}

\begin{proof}
This follows from Jensen's inequality as $(x,y) \mapsto \frac{x^2}{y}$ is jointly convex\footnote{A bivariate function $(x,y) \mapsto f(x,y)$ is \emph{jointly convex} if for all $\lambda \in [0,1]$, and for all pairs $(x_1,y_1)$ and $(x_2, y_2)$, it satisfies the inequality $f(\lambda x_1 + (1-\lambda)x_2, \lambda y_1 + (1-\lambda) y_2) \leq \lambda f(x_1, y_1) + (1-\lambda) f(x_2, y_2)$.} (see \cite[\S 3.2.6]{convexopt}).
\end{proof}

\begin{proof}[Proof of Theorem~\ref{thm:dimCDRhell}.]
The dimensional Poincar\'e inequality~\eqref{eq:dimpoincare} for the $(d-r)$ dimensional standard Gaussian measure~$\mu_\perp$ states that
\begin{equation}
\label{eq:tmpdimpoincare}
\int f^2 \d\mu_\perp - \left( \int f \d\mu_\perp \right)^2
\leq \int \| \nabla_{x_\perp} \ln f \|^2 f^2 \d\mu_\perp - \int \frac{(1 - \frac{1}{f}\int f \d\mu_\perp - \langle \nabla_{x_\perp} \ln f, x_\perp \rangle)^2}{d-r+\|x_\perp\|^2}f^2 \d\mu_\perp
\end{equation}
for any sufficiently smooth function $(x_r, x_\perp) \mapsto f(x_r, x_\perp)$ with fixed~$x_r \in \R^r$ and fixed orthonormal sub-matrix $U_r \in \R^{d \times r}$. Choose
\[
f(x_r, \x_\perp) = \sqrt{\ell(U_rx_r + U_\perp x_\perp)} \;,
\]
where $U_\perp$ is any orthonormal complement of~$U_r$ and $\ell: \R^d \to \R^+$ is such that $d\pi = \ell \d\mu$. {Without loss of generality, we can choose~$\ell$ normalized so that $\int \ell \d\mu = 1$ since the dimensional Poincar\'e inequality~\eqref{eq:dimpoincare} is homogeneous with respect to multiplicative scalings of~$f$.}
For this choice of test function~$f$, we next integrate both sides of the inequality with respect to~$x_r$ with law~$\mu_r$ given by the $r$-dimensional isotropic Gaussian measure. The left hand side of the inequality in~\eqref{eq:tmpdimpoincare} then becomes
{
\begin{align*}
\E_{x_r \sim \mu_r}\left[ \int f^2 \d\mu_\perp - \left( \int f \d\mu_\perp \right)^2
\right] &= \int f^2 \d\mu - \int \left( \int f \d\mu_\perp \right)^2 \d\mu_r \\
&\overset{\eqref{eq:optHell}}{=} 1 - \int \ell_r^\text{Hell}\d\mu_r  \\
&\overset{\eqref{eq:dhelloptval}}{=}  1 - (1-d_\mathrm{Hell}^2(\pi\,,\,\tilde\pi_r^{\mathrm{Hell}}(U_r)))^2\;,
\end{align*}
where we used the fact that $\int f^2 \d\mu = 1$ by construction. The first term in right-hand side of~\eqref{eq:tmpdimpoincare} becomes
\[
\E_{\mu_r} \left[ \int \| \nabla_{x_\perp} \ln f \|^2 f^2 \d\mu_\perp\right] =
\int \|\nabla_{x_\perp} \ln f\|^2 f^2 \d\mu = \int \frac{1}{4}\| U_\perp \nabla \ln \ell \|^2 \d\pi,
\]
after applying the change of measure $\d\pi = f^2 \d\mu$ and the chain-rule to the gradient. Meanwhile, applying Lemma~\ref{thm:jointconvexity} to the second term in the right-hand side of~\eqref{eq:tmpdimpoincare}, we obtain
\begin{align*}
\E_{\mu_r} &\left[ - \int \frac{(1 - \frac{1}{f}\int f \d\mu_\perp - \langle \nabla_{x_\perp} \ln f, x_\perp \rangle)^2}{d-r+\|x_\perp\|^2}f^2 \d\mu_\perp  \right]
= - \int \frac{(1 - \frac{1}{f}\int f \d\mu_\perp - \langle \nabla_{x_\perp} \ln f, x_\perp \rangle)^2}{d-r+\|x_\perp\|^2} \d\pi \\
&\overset{Lemma~\ref{thm:jointconvexity}}{\leq} -\frac{(1 - \E_{\pi}[\frac{1}{\sqrt{\ell}} \int \sqrt{\ell} \d\mu_\perp] - \frac{1}{2}\E_{X \sim \pi}[\langle U_\perp^\top \nabla \ln \ell, U_\perp^\top X \rangle)^2 ]}{d - r + \E_{X \sim \pi}[ \|U_\perp^\top X\|^2]} \;,
\end{align*}
where once again we have applied the chain rule and the change of measure.
The numerator can be further simplified, first by noting that
\[
\E_{\pi}\left[ \frac{1}{\sqrt{\ell}}\int \sqrt{\ell}\d\mu_\perp \right] \overset{\d\pi = \ell \d\mu}{=} \int \int \left(\int \sqrt{\ell}\d\mu_\perp \right) \sqrt{\ell} \d\mu_\perp \d\mu_r  \overset{\eqref{eq:optHell}}{=} \int \ell_r^\Hell \d\mu_r \overset{\eqref{eq:dhelloptval}}{=} (1-d_\mathrm{Hell}^2(\pi\,,\,\tilde{\pi}_r^\mathrm{Hell}(U_r)))^2 \;.
\]
Secondly, we also note that
\[
\int \nabla \ln \ell(x) x^\top \d\pi(x) \overset{\d\pi=\ell\d\mu}{=} \int \nabla \ell(x) x^\top \d\mu(x) = \int xx^\top \d\pi(x) - I,
\]
where the second equality stems from applying Stein's identity with the standard Gaussian~$\mu$, since
\[
\int \frac{\partial \ell}{\partial x_i}(x) x_j \,\d\mu(x) = \int (x_i \ell(x) - x_i \delta_{ij}) x_j \d\mu(x) \overset{\d\pi=\ell\d\mu}{=} \int x_i x_j \d\pi(x) - \delta_{ij} \;,
\]
for any $1 \leq i,j \leq d$, where $\delta_{ij} = 1$ if $i=j$ and $0$ otherwise.
}
\end{proof}

\section*{Acknowledgments}
ML and YM thank Brent Wheelock, Dimitar Trenev, Kevin Daly, Grant Seastream, and Tuan  Tran at ExxonMobil for the code and discussions related to Bayesian inference for geophysical applications with generative modeling priors.
ML also thanks Nicolas Boumal for discussions on the ManOpt forum related to optimization on matrix manifolds. {We also thank the two anonymous reviewers for their invaluable feedback which immensely improved the quality of this manuscript. In particular, we wish to thank them for pointing out references~\cite{stamInequalitiesSatisfiedQuantities1959a, Dembo_1990}, as well as for correcting numerous grammatical and mathematical typos.}

\section*{Funding}
ML and YMM acknowledge support from the US Department of Energy, Office of Advanced Scientific Computing Research, under grant DE-SC0023187 and from the ExxonMobil Technology and Engineering Company. YMM further acknowledges support from the US Department of Energy, Office of Advanced Scientific Computing Research, under grants DE-SC0021226 and DE-SC0023188. TC is supported by the ARC Discovery Project DP210103092. OZ acknowledges support from the ANR JCJC project MODENA (ANR-21-CE46-0006-01).

{
\bibliographystyle{siam}
\bibliography{references}
}

\end{document}